\documentclass{article}

\PassOptionsToPackage{numbers, compress}{natbib}

 \usepackage[preprint]{neurips_2026}
 \usepackage[utf8]{inputenc} %
\usepackage[T1]{fontenc}    %
\usepackage{hyperref}       %
\usepackage{url}            %
\usepackage{booktabs}       %
\usepackage{amsfonts}       %
\usepackage{nicefrac}       %
\usepackage{microtype}      %
\usepackage{xcolor}         %
\usepackage{amsthm}        %
\newtheorem{lemma}{Lemma}

\usepackage{graphicx} %
\usepackage{amsmath,amssymb}
\usepackage{caption}
\usepackage{subcaption}

\usepackage{tikz}
\usetikzlibrary{tikzmark}
\usetikzlibrary{bayesnet}
\usepackage{multirow}
\usepackage[dvipsnames, table]{xcolor}
\usepackage{adjustbox}
\usepackage{listings}
\usepackage{pifont}
\usepackage{wrapfig}
\usepackage{etoolbox}
\usepackage[normalem]{ulem}
\definecolor{lightgray}{gray}{0.9}

\usepackage{hyperref}
\usepackage{url}
\usepackage[english]{babel}
\usepackage[mathscr]{eucal}
\usepackage{xcolor}
\usepackage{bbding}

\definecolor{frozenblue}{RGB}{100, 181, 205}
\definecolor{sigmayellow}{RGB}{250, 243, 201}
\newcommand{\frozen}{\textcolor{frozenblue}{\tiny{\Snowflake}}}

\usepackage{amssymb}
\usepackage{tcolorbox}
\newcommand{\stds}[1]{\footnotesize{$\pm$#1}}

\title{The Key to Going Linear: Analysis-Driven Transformer Linearization}

\author{%
Anna Kuzina\\
Qualcomm AI Research\thanks{Qualcomm AI Research is an initiative of Qualcomm Technologies, Inc.} \\
{\texttt{akuzina@qti.qualcomm.com}}
\And
Paul N. Whatmough\\
Qualcomm AI Research \\
{\texttt{pwhatmou@qti.qualcomm.com}}
\And
Babak Ehteshami Bejnordi \\
Qualcomm AI Research \\
{\texttt{behtesha@qti.qualcomm.com}}
}

\date{February 2026}

\let\oldparagraph\paragraph
\renewcommand{\paragraph}[1]{\vspace{-0.3em}\oldparagraph{#1}}

\begin{document}

\maketitle

\begin{abstract}

The quadratic cost of causal self-attention severely bottlenecks long-context transformer inference. While numerous post hoc linearization pipelines exist, it is difficult to identify which components preserve model quality. This work isolates the effect of state update design in a strict frozen-backbone regime. We show that softmax relies on key-dependent, rank-1 orthogonal projections, elucidating why delta-style networks outperform purely gated accumulation. We identify a potential source of approximation errors and introduce structural interventions, specifically sink tokens, short convolutions, and fixed-budget cache routing, which reduces the remaining gap. We scale this linearization approach across LLaMA and Qwen models up to 32B parameters, outperforming prior post hoc baselines on MMLU and matching the long-context retrieval of complex adaptive-caching frameworks.

\end{abstract}

\section{Introduction}

Causal self-attention is the main obstacle to scaling pretrained transformers to longer contexts. Its quadratic compute and growing KV cache make inference increasingly expensive as sequence length grows. A particularly appealing solution is \textit{post hoc linearization}: converting an existing full-attention model into a linear-time architecture without pretraining from scratch~\cite{zhang2024lolcats,lan2025liger,mcdermott2025lola,meng2026still,van2025lizard,mercat2024linearizing}.
While recent approaches achieve this, they typically combine multiple interventions simultaneously, such as low-rank adaptation (LoRA), sliding-window attention (SWA), hybrid routing, and distillation. These confounding factors obscure a fundamental design question: 
\textbf{which linear state update best approximates pretrained softmax attention?}

To answer this, we take an analysis-driven approach to transformer linearization. We isolate the approximation problem by studying attention replacement in a strict regime: the pretrained backbone remains entirely frozen. Only the newly introduced parameters of the replacement mechanism are trained. This setting isolates the capacity of the linear attention replacements. 

Within this controlled setting, we derive a first-order approximation connecting softmax attention to linear state updates. Our analysis proves that softmax naturally applies key-dependent, rank-1 orthogonal projections. Delta-style updates (like Gated Delta Networks) natively implement these corrections through rank-1 operators constructed from keys, suppressing components aligned with previous keys. Conversely, pure gated accumulation (like Gated Linear Attention) relies on key-independent decay, missing this vital geometric correction.

We empirically validate these insights across LLaMA and Qwen backbones. We first compare normalized kernelized linear attention~\cite{zhang2024hedgehog}, Gated Linear Attention (GLA)~\cite{yang2023gated}, and Gated Delta Networks~\cite{team2025kimi} in the strict frozen-backbone setting. Delta-style updates consistently provide the strongest approximation. Guided by a theoretically identified approximation gap, we introduce practical design choices: a disjoint sliding-window path, sink tokens, and  projection adaptation through short convolutions or LoRA. We find these components compensate for the limitations of linear mechanisms, narrowing the performance gap and serving as a robust default for post hoc linearization. Our main contributions are as follows:

\begin{itemize}
    \item We study post hoc linearization of pretrained language models in a strict frozen-backbone setting, and empirically compare kernelized, gated, and delta-based linear attention mechanisms as replacements for causal self-attention.
    \item We provide a first-order approximation proving delta-based mechanisms faithfully capture the key-dependent, rank-1 dynamics of softmax attention, unlike pure gated updates.
    \item We empirically demonstrate how practical interventions including SWA, sink tokens, and limited Q/K/V adaptation (e.g. short convolutions) compensate for approximation gaps.
    \item We outperform prior post hoc baselines on MMLU, match the long-context retrieval of complex adaptive token-caching frameworks, and remain competitive across all downstream tasks, all while demonstrating effectiveness in hybrid settings and consistent scaling across LLaMA and Qwen models up to 32B parameters.   
\end{itemize}

\section{Background and Related works}

\textbf{Linear-Time Sequence Models.} We build upon linear-time models that serve as primitives for transformer linearization. Early kernelized approaches, such as Performer~\cite{choromanski2020rethinking} and later softmax-mimicking designs such as Hedgehog/Porcupine~\cite{zhang2024hedgehog} develop feature-map based approximations to softmax attention, while recurrent and state-based models such as Gated Linear Attention (GLA)~\cite{yang2023gated}, delta-rule linear transformers~\cite{yang2024parallelizing}, RetNet~\cite{sun2023retentive}, Longhorn~\cite{liu2024longhorn}, Gated Delta Networks (GDN)~\cite{yang2024gated}, Kimi Linear~\cite{team2025kimi}, and transformer-SSM views of Mamba-style models~\cite{dao2024transformers, egorov2025myosotis} provide increasingly expressive state updates for long-context modeling. These works motivate the replacement mechanisms studied here, but they are typically introduced as standalone architectures or pretraining recipes rather than analyzed as drop-in substitutes for frozen pretrained attention blocks.

\textbf{Post-Hoc Linearization of Pretrained Transformers.} Our work directly relates to post-hoc linearization, which converts full-attention transformers into subquadratic models for efficient inference. To recover the performance lost during this conversion, existing approaches frequently employ various structural interventions. These include intra-layer hybrid attention mechanisms, where Sliding Window Attention (SWA) is combined with the linear path to preserve essential local context, as seen in LoLCATs~\cite{zhang2024lolcats} and Liger~\cite{lan2025liger}. Some works also rely on sparse caching strategies such as attention sinks and hybrid token selection in LoLA~\cite{mcdermott2025lola}, STILL~\cite{meng2026still}, and Lizard~\cite{van2025lizard}, to maintain exact long-range retrieval. Other common components include adaptive distillation~\cite{ro2025fly}, direct layerwise linearization~\cite{mercat2024linearizing}, and low-rank adaptation (LoRA) on the original Q, K, V projections~\cite{zhang2024lolcats,mcdermott2025lola}.

While these interventions reduce the downstream performance gap, combining them obscures the approximation capacity of linear state updates. We therefore evaluate linear updates in a frozen‑backbone setting to isolate their limitations. After identifying these gaps both theoretically and empirically, we reintroduce SWA, sink tokens, short convolutions, and hybrid routing to assess their contribution.

\section{Full Attention Approximation via Linear Mechanisms}

We study the problem of approximating pretrained causal self-attention using linear-time mechanisms while minimizing architectural changes and finetuning overhead.
Consider a pretrained transformer with a causal attention block, 
our goal is to replace each attention block with a mechanism that has linear complexity with respect to the number of tokens, while preserving downstream task performance as much as possible.

\begin{figure}[t]
    \centering
        \includegraphics[width=0.99\linewidth]{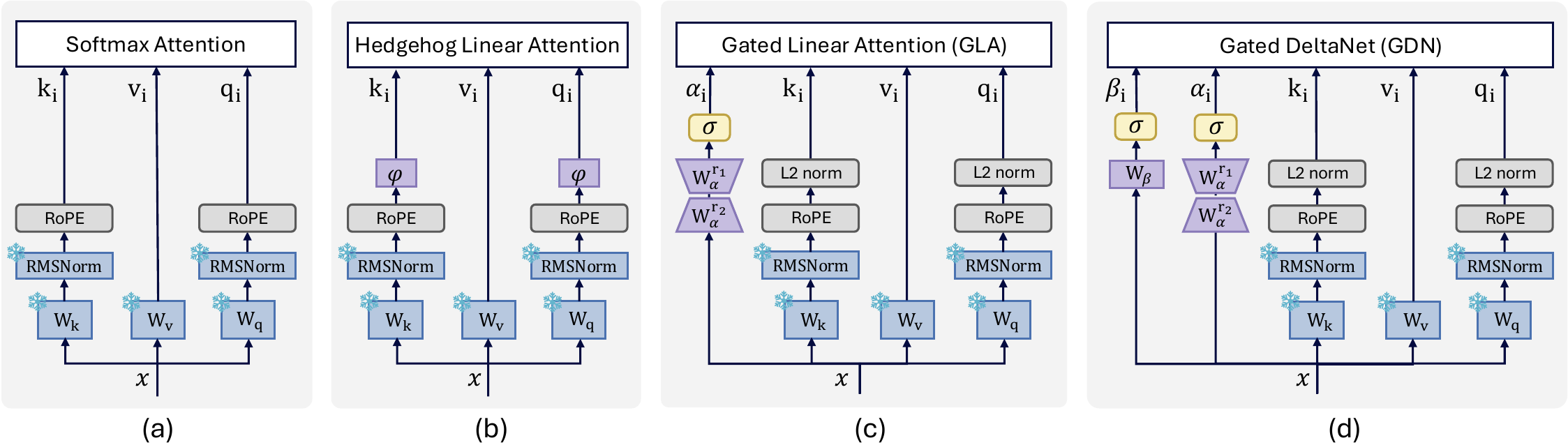}
        \vskip -5pt
        \caption{\textbf{Attention replacement blocks evaluated in our strict frozen-backbone setting.} Blue \frozen{} markers denote frozen pretrained parameters. The pretrained parameters including $W_q$,$W_k$,$W_v$, and RMSNorm parameters are kept frozen, while only the newly introduced mechanism-specific parameters are trained. (a) Standard Softmax. (b) Hedgehog~\cite{zhang2024hedgehog} replaces the softmax kernel with learned feature maps $\varphi$. (c) GLA~\cite{yang2023gated} introduces an input-dependent decay $\alpha_i$. (d) GDN~\cite{team2025kimi} applies a key-dependent rank-1 delta update via $\beta_i$. \tikzmarknode[draw,inner sep=2pt,rounded corners,fill={rgb,255:red,250; green,243; blue,201}]{A}{$\sigma$} denotes sigmoid function.}\label{fig:att_mechanisms}
        \vskip -15pt
\end{figure}

\subsection{Overview of Linear Attention Mechanisms}

We consider two families of linear mechanisms used for attention replacement:
\emph{normalized kernel attention} and \emph{state-based unnormalized attention}. In Figure~\ref{fig:att_mechanisms} we present as overview of the linear attention mechanism we use for linearization of the Softmax Attention. 
Recall that Softmax Attention can be written as 
\begin{equation}
y_t
=
\sum_{i \le t}
\frac{\exp(q_t^\top k_i / \sqrt{d})}
{\sum_{j \le t} \exp(q_t^\top k_j / \sqrt{d})}
\, v_i .
\end{equation}
The exponential kernel $\exp(q^\top k)$ can be viewed as a similarity function
$\mathcal{K}(q, k)$, motivating kernel-based approximations.

\subsubsection{Kernelized Linear Attention}
Following prior work on kernelized linear attention, we approximate the softmax
kernel using feature maps $\phi : \mathbb{R}^d \to \mathbb{R}^{d'}$ such that $\exp(q^\top k) \approx \phi(q)^\top \phi(k).$
Substituting this into attention yields
\begin{align}
y_t = \frac{\phi(q_t)^\top S_t}{\phi(q_t)^\top Z_t}, \qquad S_t = \sum_{i \le t} \phi(k_i) v_i^\top,
\qquad
Z_t = \sum_{i \le t} \phi(k_i),
\label{eq:normalized-linear-attention}
\end{align}
This form \emph{explicitly normalizes} attention weights. The denominator
$Z_t$ acts as a running partition function shared across values.

\paragraph{Hedgehog.}
Following prior linearization works~\citep{zhang2024lolcats}, we consider Hedgehog~\citep{zhang2024hedgehog} feature maps. 
\begin{equation}
   \phi(x) = \sigma(xW) \oplus \sigma(-xW),
\end{equation}
where $x$ is keys or queries vector, $\sigma$ denotes softmax and $\oplus$ is a concatenation operator.

\subsubsection{State-Based Linear Attention without Normalization (GLA, GDN)}

In contrast, Gated Linear Attention (GLA) and Gated DeltaNet (GDN) do not explicitly normalize attention weights. These methods maintain a recurrent \emph{state} $S_t$ such that
\begin{equation}
y_t = q_t^\top S_t, \qquad S_t = \mathcal{U}(S_{t-1}, k_t, v_t),
\end{equation}
where $\mathcal{U}$ is a general form of a linear or affine update. However, by unrolling the recurrence, we can rewrite them in the same format. GLA~\citep{yang2023gated} updates the state using
\begin{equation}
S_t = G_t S_{t-1} + k_t^\top v_t,
\end{equation}
where the forget gate is usually chosen to be a diagonal input-dependent matrix $G_t = \mathrm{diag}(\alpha_t)$.

Gated Delta Net~\citep{yang2024gated} augments gated linear attention with a delta rule~\citep{yang2024parallelizing}:
\begin{equation}
S_t = (I - \beta_t k_t k_t^\top)G_t S_{t-1} + \beta_t k_t^\top v_t,
\end{equation}
where $\beta_t$ is an input-dependent per-head scalar. Note that originally GDN used a per-head scalar gate instead of a diagonal matrix. Kimi Linear~\citep{team2025kimi} extended it to the same form as GLA. Unless stated explicitly, we will refer to the Kimi version of GatedDeltaNet as GDN.

\subsection{From Softmax to Delta Rules: A First-Order Approximation View}
\label{sec:delta-vs-gating}

Kernelized linear attention methods replace the softmax kernel with a factorized feature-map representation, yielding a computationally linear form that often serves as a faithful approximation to softmax attention~\citep{choromanski2020rethinking, qin2022cosformer}. In contrast, for state‑space linear attention variants such as GLA and Gated DeltaNet, it remains unclear whether their update rules correspond to softmax attention, despite strong empirical pre‑training performance. In this work, we derive a first‑order approximation to softmax attention and show that its structure aligns with Gated DeltaNet, supporting its suitability for linearization.

To understand why delta-based mechanisms are a faithful approximation of
attention, we analyze attention outputs as weighted sums of values.
The attention output at time $t$ (before the output projection) can be written as a weighted sum of values
$y_t = \sum_{i \le t} P_{t,i} \, v_i,$
where $P_{t,i}$ are the causal attention weights.
For full softmax attention, they take the following form:
\begin{equation}
P^{\mathrm{full}}_{i}
=
\frac{\exp\!\big(\frac{1}{\sqrt{d}}\, q^\top k_i\big)}
{\sum_{j=1}^{t} \exp\!\big(\frac{1}{\sqrt{d}}\, q^\top k_j\big)}.
\label{eq:full-att-softmax}
\end{equation}

These weights are normalized, nonnegative, and depend jointly on the current
query and all previous keys, inducing adaptive reweighting and implicit
forgetting.

\paragraph{Gated Linear Attention (GLA).}
GLA induces implicit weights
\begin{equation}
P^{\mathrm{GLA}}_{i}
=
q^\top
\Bigl[
\prod_{j=i+1}^{t} \mathrm{diag}(\alpha_j)
\Bigr]
k_i
=
q^\top W_i k_i,
\label{eq:gla-weights}
\end{equation}
The decay matrix $W_i$ depends on recency, since is a product of vector-valued gates $\alpha_j \in (0,1)^d$, but is query- and key-independent.

\paragraph{Gated DeltaNet (GDN).}
DeltaNet induces weights
\begin{equation}
P^{\Delta}_{i}
=
\beta_i \,
q^\top W_i
\Bigl(
\prod_{j=i+1}^{t}
(I - \beta_j k_j k_j^\top)
\Bigr)
k_i,
\label{eq:gda-weights}
\end{equation}
Unlike GLA, the effective weights depend explicitly on previous keys via rank-1 operators.
Each factor $(I - \beta_j k_j k_j^\top)$ preserves components orthogonal to $k_j$ while downweighting the rest.

\subsubsection{Derivations}
\label{subsec:softmax-to-delta}

 We start from the softmax attention weights $P^{\mathrm{full}}_{i}$ and aim to approximate them with something of the form $q^\top A k_i$ for some matrix $A$, linking it to linear attention.

\begin{lemma}[First--order Taylor approximation of softmax]
Let
$p(x) = \mathrm{softmax}(x)$,
$x_i = q_t^\top k_i$, and
$\bar{x} = \frac{1}{t}\sum_{j=1}^t x_j.$
Then for each index $i$,
\[
p_i(x)
=
\frac{1}{t}
+
\frac{1}{t}\bigl(x_i - \bar{x}\bigr)
+
\varepsilon_i^{(1)}(x),
\]
where the error term satisfies
\[
\|\varepsilon^{(1)}(x)\|_\infty \le \frac{1}{2}\|x\|_2^2 .
\]
\end{lemma}
See Appendix~\ref{app:proof1} for the proof.
Substituting $x_i = q_t^\top k_i / \sqrt{d}$ gives
$
p_i(x)
\approx
\frac{1}{t}
+
\frac{1}{t\sqrt{d}}q_t^\top (k_i-\bar{k}).
$

\begin{lemma}[Projection approximation]
Define the mean key and the projection matrix
\[
\bar{k}=\frac{1}{t}\sum_{j=1}^t k_j,
\qquad
\Pi_{\bar{k}}=\frac{\bar{k}\bar{k}^\top}{\|\bar{k}\|^2}.
\]
Assume all keys are normalized, $\|k_i\|=1 \quad \text{for all } i$. Then
\begin{equation}\label{eq:residual}    
k_i-\bar{k}=(I-\Pi_{\bar{k}})k_i+r_i,
\end{equation}
where the residual satisfies $
\|r_i\|
\le
|\cos(\alpha_i)-1| + |1-\|\bar{k}\||,
$
and $\alpha_i$ is the angle between $k_i$ and $\bar{k}$.
\end{lemma}

Here, we replace subtraction of the average with the projection orthogonal to the average key. This introduces an error that vanishes when keys are directionally aligned
($\alpha_i=0$) and concentrated around their mean
($\|\bar{k}\|\approx1$). See App.~\ref{app:proof2} for the proof.

Combining these two steps, we showed that $P^{\mathrm{full}}_{i} \approx \frac{1}{t\sqrt{d}}q^\top (I-\Pi_{\bar{k}})k_i + \frac{1}{t}$.

\begin{lemma}[Rank--1 product approximation]
Let $\|k_j\|=1$ and define
\[
M_t = \prod_{j=1}^t D_j,\quad \text{where }
D_j = I - \beta_j k_j k_j^\top,
\text{ and }
\beta_j \in (0,1),
\]
Then
\[
M_t
=
I
-
\sum_{j=1}^t \beta_j k_j k_j^\top
+
E_t,
\qquad
\|E_t\|_2 \le C\, t (\max_j \beta_j)^2.
\]
If the vectors $k_j$ are concentrated around their mean
$\bar{k}=\frac{1}{t}\sum_{j=1}^t k_j$, then
\[
M_t \approx I - c\,\Pi_{\bar{k}},
\quad
\Pi_{\bar{k}}=\frac{\bar{k}\bar{k}^\top}{\|\bar{k}\|^2},
\]
for some scalar $c \ge 0$.
\end{lemma}
See App.~\ref{app:proof3} for the proof. This final approximation steps allow us to replace a single projection to the direction orthogonal to the mean key with the sequence of projection that depend on all the keys. Combining it with the previous result, we obtain the following approximation to the full attention weights:
\begin{equation}
    P^{\mathrm{full}}_{i} \approx \frac{1}{t\sqrt{d}}q^\top \prod_j (I-\beta_j k_jk_j^\top)k_i + \frac{1}{t},
\end{equation}
which aligns with the linear attention weights we get when using delta rule (Eq.~\ref{eq:gda-weights}).
Note that while the Taylor expansion yields a constant bias $\frac{1}{t}$, practical GDN implementations do not have it, focusing entirely on capturing the dynamic, key-dependent variance. For large $t$ values, this becomes negligible; however, we expect it to lead to higher approximation error at the beginning of the sequence.

\subsubsection{Gated Linear Attention with Key-Dependent Gate}
These approximations demonstrate that delta-based updates capture the first-order dynamics of softmax attention, while purely gated accumulation does not.
However, we can potentially reduce the gap between the two linear kernels by introducing key‑dependent gating. Standard GLA attention (defined in Eq.~\ref{eq:gla-weights}) computes a product of diagonal matrices, which only depends on the hidden states $h_j$ through the gates $\alpha_j = \sigma(W^g h_j)$. 
To make GLA a better approximation of full attention, while preserving its hardware-efficient diagonal structure, we suggest using the following form:
\begin{equation}
    y_t = \sum_i (q^\top \big[\prod_{j=i+1}^t\operatorname{diag}(\alpha_j) \Big( I - \beta_j\,\operatorname{Diag}\big( k_j \odot k_j\big)\Big) \big] k_i)v_i
\end{equation}
where we use the same $\beta_j$ as in GDN.
We refer to this variant as \textit{kGLA} (Key-Gated Linear Attention).

While this diagonal approximation $\operatorname{Diag}(k_j \odot k_j)$ is inherently less expressive than the exact rank-1 update $k_j k_j^\top$ used in the delta rule, it may serve as a stronger structural prior for GLA models and allows us to leverage existing highly optimized GLA kernels.

\subsection{Empirical Evaluation}
\paragraph{Setup.}
We start by evaluating linear blocks without additional components like SWA or LoRA. We replace all layers in our full-attention baselines (\textsc{Llama3.1-8b} and \textsc{Qwen3-8b-Base}) with corresponding linear blocks.
We use the \texttt{
FLA}~\footnote{\texttt{Flash Linear Attention~\citep{yang2024fla}}}
library for efficient GLA and GDN kernel implementations. For Hedgehog,  we use \texttt{FLA} implementation of the feature maps $\phi$ and implement its chunk-wise-parallel form for training and recurrent form for inference in \texttt{pytorch}. We freeze all parameters of the original model and train only the newly added parameters. We show the parametrization of the linear attention models as well as frozen and trainable parameters in Figure~\ref{fig:att_mechanisms}. See App.~\ref{app:hyperparams} for the full list of hyperparameters and implementation details. 

We train the model end-to-end with the standard next-token prediction loss (cross-entropy) on the DCLM-Edu~\citep{li2024datacomp, allal2025smollm2smolgoesbig} dataset for 10M tokens. We pack sequences to 4096 tokens and perform 2,500 training iterations, with batch size 1. Training GLA and GDN variants takes 1 hour on a single H100 GPU. Training Hedgehog takes approximately 8 hours, due to the absence of a Triton implementation.

Note that this setting is intentionally stricter than most prior linear-attention work: we do not perform a distillation step, do not modify the original Q, K, V projections, and do not add sliding-window attention. This allows us to isolate the approximation properties of the linear mechanisms themselves.

\begin{figure}[t] %
\centering
\begin{subfigure}{0.3\linewidth}
  \centering
    \includegraphics[width=0.9\linewidth]{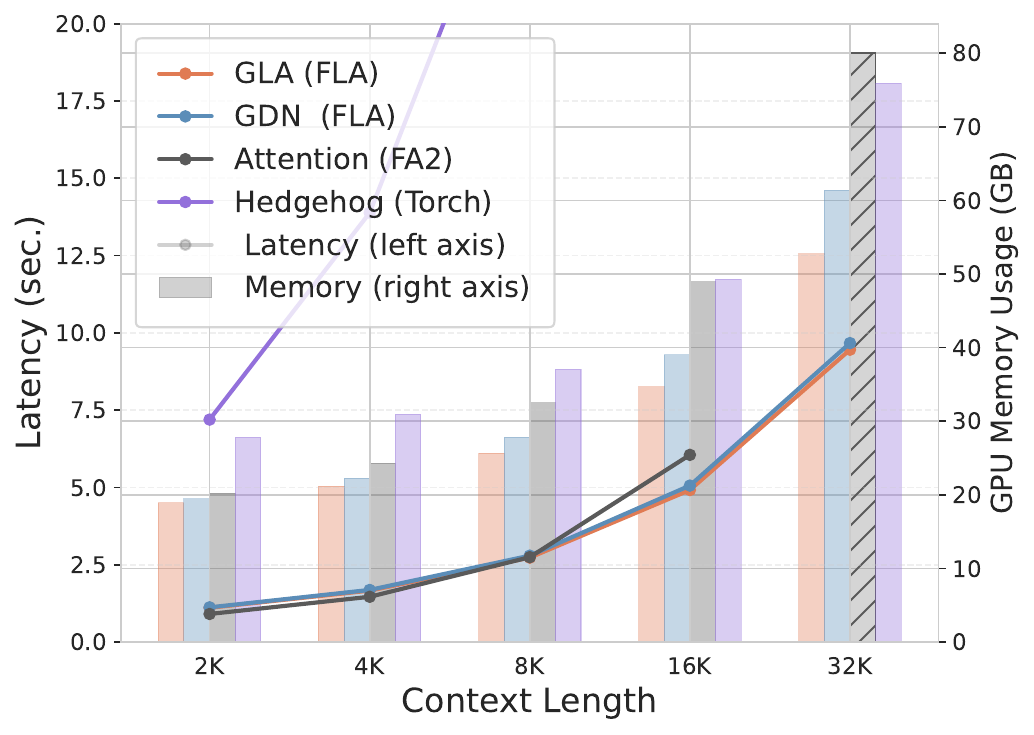} 
    \vskip -5pt
  \caption{Latency}
\end{subfigure}\hfill
\begin{subfigure}{0.3\linewidth}
  \centering
  \includegraphics[width=0.9\linewidth]{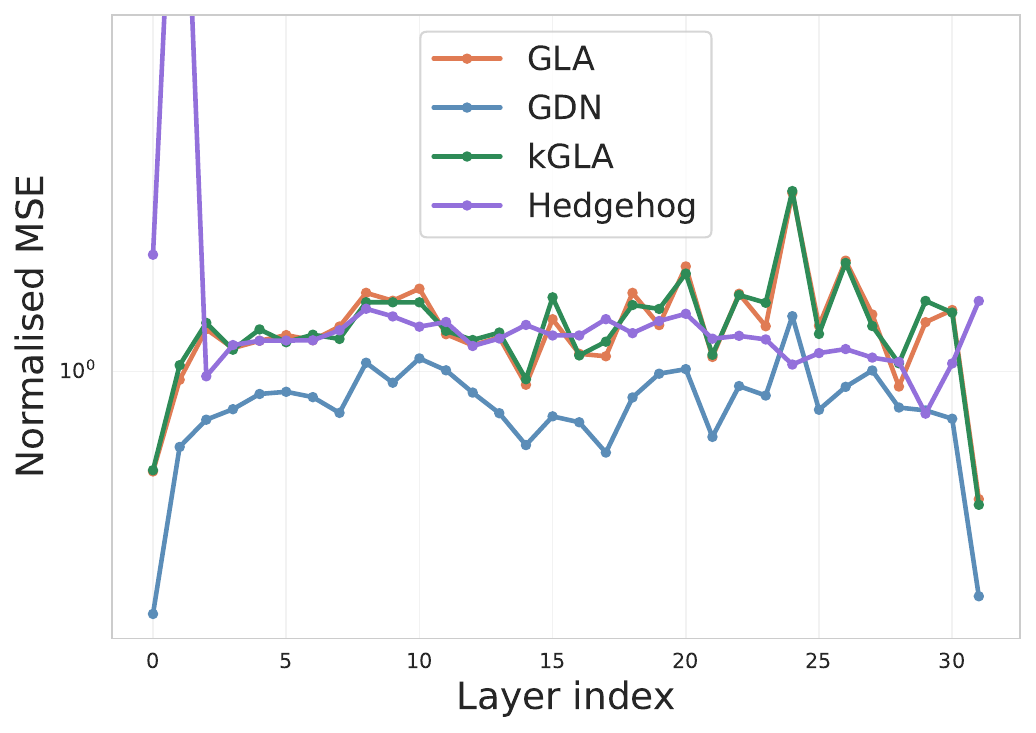}
  \vskip -5pt
  \caption{MSE Per Layer}
\end{subfigure}\hfill
\begin{subfigure}{0.3\linewidth}
  \centering
  \includegraphics[width=0.9\linewidth]{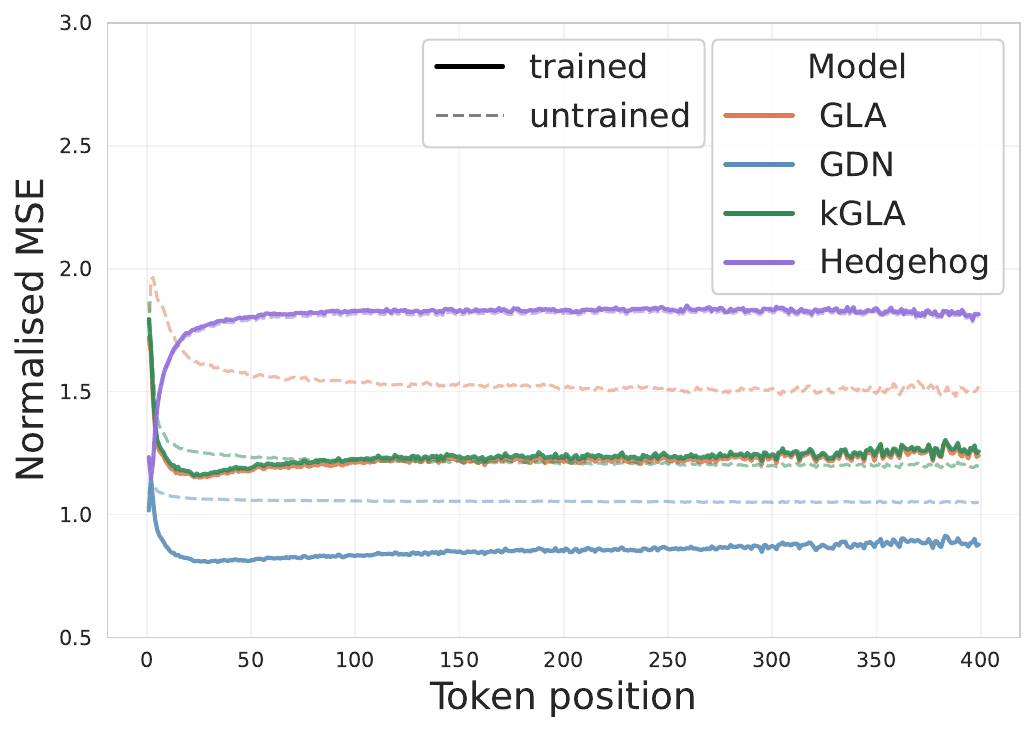} 
  \vskip -5pt
  \caption{MSE per Token}
\end{subfigure}
\vskip -5pt
\caption{Comparison of linear attention variants for Llama3.1-8b linearization.}\label{fig:pure_mse_latecy}
\vskip -10pt
\end{figure}
\begin{figure}[t] 
\centering
\begin{subfigure}{0.62\linewidth}
  \centering
  \includegraphics[width=0.9\linewidth]{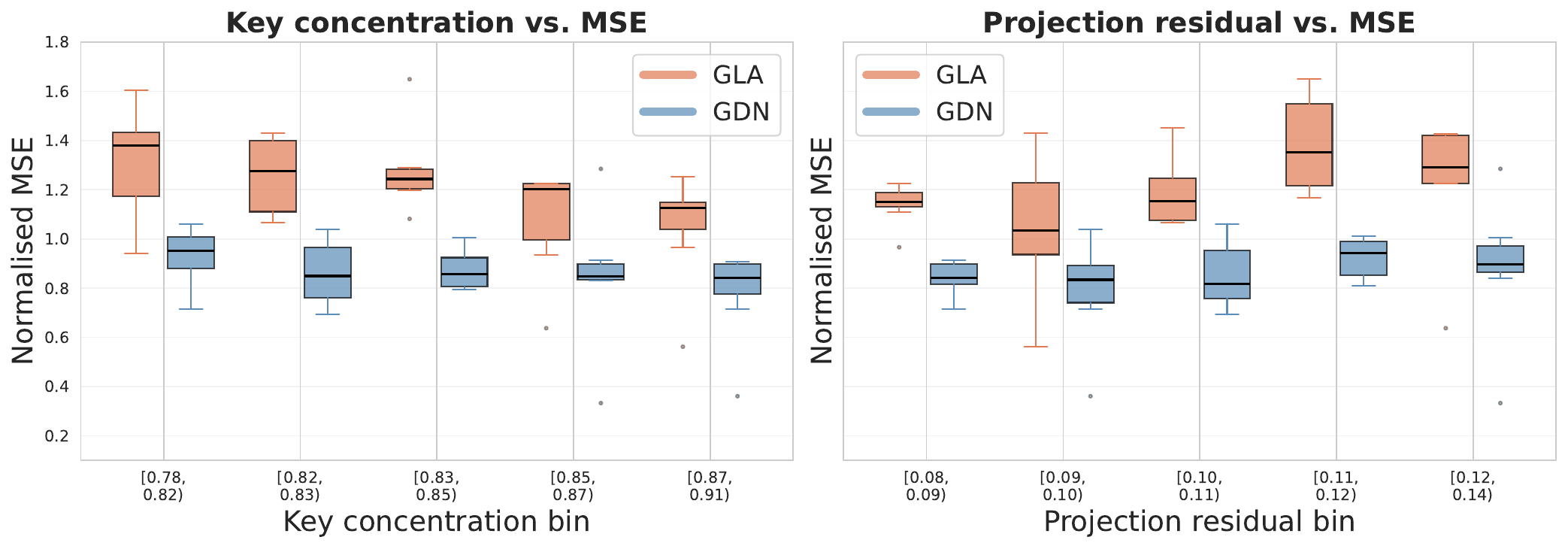} 
  \vskip -5pt
            \caption{Per layer normalized MSE}
\end{subfigure}\hfill
\begin{subfigure}{0.3\linewidth}
  \centering
\includegraphics[width=0.9\linewidth]{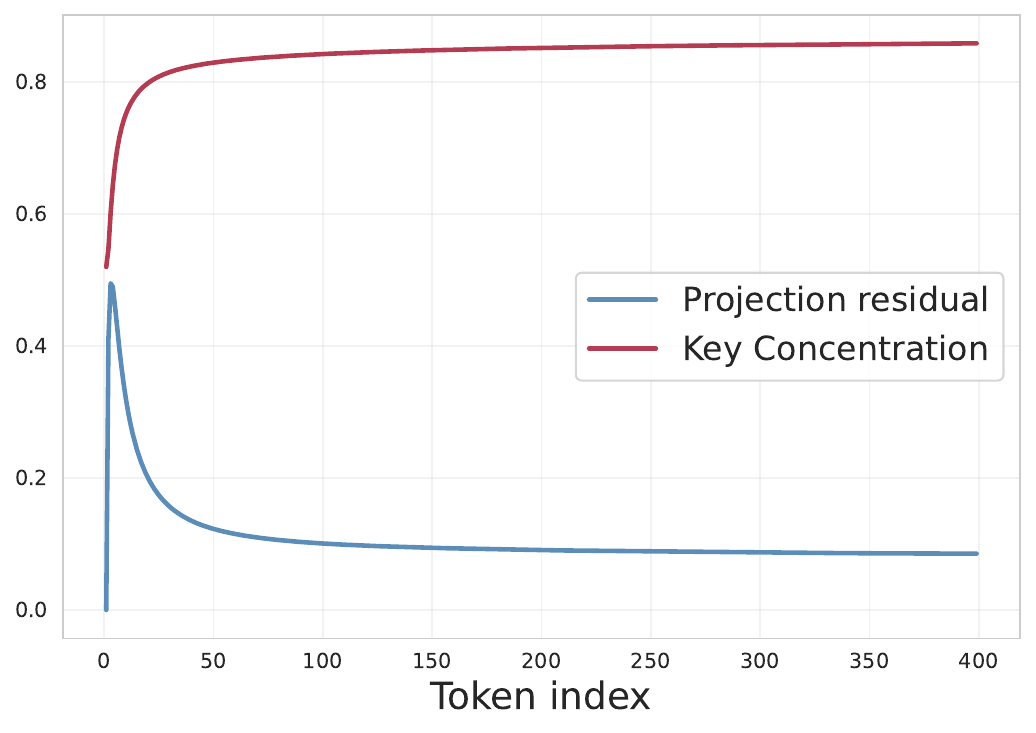} 
\vskip -5pt
    \caption{Token index}
\end{subfigure}
\vskip -5pt
\caption{Key Concentration and Projection Residual for LLama3.1-8b}\label{fig:lemmas}
\vskip -10pt
\end{figure}

\paragraph{Latency, Memory, and Approximation Error.}
Figure~\ref{fig:pure_mse_latecy}(a) shows latency and memory as a function of context size. Linear attention yields substantial memory savings, enabling evaluation up to 32k tokens, whereas \texttt{FA2} runs out of memory in our setup. Latency gains are modest below $\sim$8k tokens, consistent with optimized quadratic kernels remaining competitive at short contexts.

Figures~\ref{fig:pure_mse_latecy}(b--c) report the layerwise and tokenwise normalized MSE between the frozen full-attention outputs and the linearized replacements. GDN consistently achieves lower MSE than other linear kernels. As predicted by our first-order analysis, error is largest early in the sequence: the softmax Taylor expansion includes a uniform bias term $1/t$ that standard GDN/GLA implementations omit, making the first few timesteps harder to match when $t$ is small. 
While this is not identical to the \emph{attention sink} effect~\citep{xiao2023efficient}, it reflects a related observation: decoder-only LMs often place disproportionately high attention mass on the first tokens, producing sharp prefix attention patterns that are difficult for linear mechanisms to approximate. We refer to these prefix positions as \emph{sink tokens} throughout.

Finally, Key-Gated Linear Attention (green) has lower MSE \textit{at initialization} than GLA (orange), as shown by the dashed curves, but this gap largely disappears after training. This suggests lightweight fine-tuning mitigates the initialization advantage and that performance is ultimately capped by the diagonal gate structure. In contrast, the rank-1 update in GDN yields both better initialization and better final approximation. See App.~\ref{app:pure-linear-detailed} for full implementation details and Qwen3-8B results.

\paragraph{Assumption Checks: Residuals and Key Concentration.}
We empirically evaluate the assumptions used in Sec.~\ref{sec:delta-vs-gating} in Fig.~\ref{fig:lemmas} for the Llama3.1-8B base model. See App.~\ref{app:pure-linear-detailed} for full implementation details and Qwen3-8B results.

First, we report the \emph{projection residual} (Lemma~2; Eq.~\ref{eq:residual}), which quantifies the error incurred when replacing mean subtraction $(\tilde{k}_i-\bar{k}_t)$ with orthogonal projection $(I-\Pi_{\bar{k}_t})\tilde{k}_i$. As shown in Fig.~\ref{fig:lemmas}(a), larger residuals correlate with higher approximation error (normalized MSE). Fig.~\ref{fig:lemmas}(b) further shows that residual is largest at the beginning of the sequence, consistent with early tokens being hardest to approximate.
Second, we measure \emph{key concentration} (Lemma~3) by computing the running mean of normalized keys and taking the cosine similarity between each key and the (normalized) running mean direction. Higher concentration corresponds to keys aligning with a dominant direction, the regime in which the rank-1 product in Lemma~3 behaves like a structured projection. We again observe a favorable trend: lower concentration is associated with higher MSE, and the first tokens exhibit the lowest concentration.

\begin{table*}[t]
\centering
\vskip -10pt
\caption{Downstream performance of the linearized models without SWA and fully fixed Q, K, V. Averaged over 3 random seeds. See App.~\ref{app:pure-linear-detailed} for detailed results.}
\vskip -5pt
\label{tab:downstream_pure}
\resizebox{\linewidth}{!}{
\begin{tabular}{lcccccccc}
\toprule
 & \multicolumn{4}{c}{Llama3.1 8b} & \multicolumn{4}{c}{Qwen3 8b} \\
\cmidrule(lr){2-5} \cmidrule(lr){6-9}
 & Param & CR & MMLU & Lambada & Param & CR & MMLU & Lambada \\
\midrule  \rowcolor{lightgray}
Full Attention & --- & 73.98 & 65.25 & 74.79 & ---&  73.73 & 76.94 & 70.95 \\
GLA      & 4.3M  & 61.81 \stds{0.07} & 27.13 \stds{0.47} & \underline{28.26} \stds{0.11}  & 4.7M & 52.04 \stds{0.02} & 24.50 \stds{0.07} & 11.29 \stds{0.07}   \\
kGLA     & 8.5M &  \underline{62.60} \stds{0.02} & \underline{27.37} \stds{0.09} & 27.49 \stds{0.31}  & 9.6M  &  \underline{52.76} \stds{0.09} & \underline{25.23} \stds{0.23} & \underline{11.33} \stds{0.19} \\
GDN      & 8.5M  & \textbf{63.08} \stds{0.09} &\textbf{27.55} \stds{0.36} & \textbf{44.67} \stds{0.26}  & 9.6M  & \textbf{56.55} \stds{0.12} & \textbf{25.57} \stds{0.23} & \textbf{23.60} \stds{0.33} \\
Hedgehog & 10.6M &  36.13 \stds{0.02} & 24.95 \stds{0.03} & 0.00 \stds{0.00}  & 11.9M & 35.90 \stds{0.06} & 22.91 \stds{0.01} & 0.00 \stds{0.00} \\
\midrule
\multicolumn{2}{l}{\textbf{Gap with Full, \%}} & \textcolor{red}{10.9} & \textcolor{red}{37.7} & \textcolor{red}{30.12} & & \textcolor{red}{17.18} & \textcolor{red}{51.37} & \textcolor{red}{47.35}\\
\midrule
\bottomrule
\end{tabular}
}
\end{table*}

\paragraph{Downstream Evaluation.}
Table~\ref{tab:downstream_pure} reports zero-shot performance via LM-Eval~\citep{sutawika2025eleutherai} on five common reasoning benchmarks (CR): Physical Interaction QA (PIQA), ARC Easy, ARC Challenge, HellaSwag, WinoGrande. We also report 5-shot MMLU and Lambada (OpenAI) accuracy. 
Notably, the \emph{state-based} mechanisms (GLA/kGLA/GDN) substantially outperform the \emph{kernelized} Hedgehog replacement.
This separation without state-space models aligns with our first-order view: delta-style updates natively implement key-dependent rank-1 corrections that match the geometry induced by softmax, whereas purely kernelized feature maps must learn an adequate similarity representation under tight adaptation constraints.

To bridge the remaining gap, prior work typically introduces structural compensations like SWA~\citep{zhang2024lolcats, lan2025liger}, and LoRA over projection layers. In the next sections, we reintroduce these components in a controlled manner to isolate which interventions most effectively compensate for the residual approximation errors observed in the pure replacement setting.

\section{Improving Linear Attention Block Design}

\begin{table*}[t]
\centering
\vskip -5pt
\caption{Downstream performance of the linearized models: SWA and Sink Token, averaged over 3 random seeds. See App.~\ref{app:swa-sink-detailed} for detailed results.}
\vskip -5pt
\label{tab:downstream_pure_swa}
\resizebox{\linewidth}{!}{
\begin{tabular}{cclcccccccc}
\toprule
 &&& \multicolumn{4}{c}{Llama3.1 8b} & \multicolumn{4}{c}{Qwen3 8b} \\
\cmidrule(lr){4-7} \cmidrule(lr){8-11}
 &&& Param & CR & MMLU & Lambada & Param & CR & MMLU & Lambada \\
\midrule  \rowcolor{lightgray}
\multicolumn{3}{l}{Full Attention} & --- & 73.98 & 65.25 & 74.79 & ---&  73.73 & 76.94 & 70.95 \\
\multirow{3}{*}{\rotatebox[origin=c]{90}{\small{SWA}}}
&&GLA      & 4.3M & 68.52 \stds{0.07} & 39.50 \stds{0.60} & 36.19 \stds{0.14}& 4.7M & 71.08 \stds{0.01} & 48.10 \stds{0.51} & 52.55 \stds{0.21}    \\
&&kGLA     & 8.5M &  68.36 \stds{0.03} & 37.63 \stds{0.02} & 36.86 \stds{0.16}   & 9.6M  &  71.13 \stds{0.05} & 48.86 \stds{0.53} & 53.30 \stds{0.22} \\
&&GDN      & 8.5M & 68.73 \stds{0.03} & 51.57 \stds{0.77} & 41.92 \stds{0.22} & 9.6M  &  71.44 \stds{0.07} & 53.70 \stds{0.20} & 59.30 \stds{0.13}\\
\midrule
\multirow{3}{*}{\rotatebox[origin=c]{90}{\small{SWA}}}
&
\multirow{3}{*}{\rotatebox[origin=c]{90}{\small{ + Sink}}}
&GLA      & 4.3M & 73.07 \stds{0.01} & 44.25 \stds{0.08} & 59.31 \stds{0.11}   & 4.7M &  73.10 \stds{0.03} & 67.40 \stds{0.05} & 56.50 \stds{0.29}   \\
&&kGLA     & 8.5M &  73.05 \stds{0.02} & 43.28 \stds{0.03} & 59.57 \stds{0.19} & 9.6M  & 73.08 \stds{0.01} & 67.70 \stds{0.17} & 57.05 \stds{0.12} \\
&&GDN      & 8.5M &  73.81 \stds{0.00} & 58.88 \stds{0.14} & 68.82 \stds{0.13} & 9.6M  & 73.46 \stds{0.01} & 70.97 \stds{0.18} & 64.53 \stds{0.04}\\
\midrule
\multicolumn{4}{l}{\textbf{Gap with Full, \%}} &\textcolor{ForestGreen}{0.17} & \textcolor{red}{6.37} & \textcolor{red}{5.97} & & \textcolor{ForestGreen}{0.27} & \textcolor{red}{5.97} & \textcolor{red}{6.42}\\
\midrule
\bottomrule
\end{tabular}
}
\vskip -10pt
\end{table*}

\paragraph{Adding Sliding Window Attention.} 
Combining linear attention with sliding-window attention (SWA) has been shown to substantially improve downstream performance of linearized models~\citep{lan2025liger, zhang2024lolcats, mcdermott2025lola, van2025lizard}.
Let $w$ be the sliding-window size and let $y_t^{\text{swa}}$ denote the SWA output for token $t$. In Liger-GLA~\citep{lan2025liger}, the total output takes the form $y_t = 0.5\, q_t S_t + 0.5\, y_t^{\text{swa}}$.
Tokens inside the window ($i \in [t-w, t]$) contribute through both the linear and SWA paths, while tokens outside the window ($i \in [1, t-w)$) contribute only through the linear path.
Instead, we use:
\begin{equation}
    y_t = q_t S_{t-w} + y_t^{\text{swa}},
\end{equation}
so each token contributes either via the linear path or via SWA (no mixing weights). This is inspired by LoLCaTs~\citep{zhang2024lolcats}. Since we use a simpler form of linear attention, we do not need to change the normalization constant, and we can compute both components independently using standard kernels: \texttt{FA2}~\footnote{\texttt{Flash Attention 2~\cite{dao2023flashattention}}} for SWA and \texttt{FLA} for the LA part.

Our analysis also shows that the first few tokens are the hardest to approximate with linear attention. We therefore split a fixed SWA budget between $w$ local-window tokens and $s$ \textit{sink tokens}. Following~\cite{zhang2024lolcats}, we fix the budget to 64 tokens and use a 56-token sliding window with 8 sink tokens. The final output is:
\begin{align}
    y_t = q_t S_{s:t-w} + y_t^{\text{swa}},\quad
    y_t^{\text{swa}}= \frac{\sum_{i \in I} \exp(q_t^\top k_i / \sqrt{d})\, v_i}{\sum_{i \in I} \exp(q_t^\top k_i / \sqrt{d})}, \quad
    \small{I = \{1, \dots, s,\, t-w, \dots, t\}.}
\end{align}
Note that $y_t^{\text{swa}}$ can be computed efficiently with \texttt{FA2}, while the linear term can be computed with existing \texttt{FLA} kernels, enabling efficient training and inference.

We report the result for three linear kernels and two backbone model in Table~\ref{tab:downstream_pure_swa}. Consistent with the prior works, adding sliding window to the linear layer reduces the gap with the full attention that we observe in Table~\ref{tab:downstream_pure}. Furthermore, keeping the cache budget fixed to 64 tokens and allocating 8 of them to the sink instead of the sliding window, improved performance even further. This result is consistent across linear attention types and backbone architectures.

\paragraph{Adapting Q/K/V.}
Performance can be improved further by allowing modifications to the original model’s Q, K, and V projections, for which we consider LoRA and short convolutions. We keep all hyperparameters fixed, except reducing the LoRA learning rate to $5\text{e-}4$ due to training instability.

We report results in Table~\ref{tab:downstream_sc_lora_swa}. We find that the gap between models with and without the delta rule almost disappears when we introduce short convolutions. A similar phenomenon was observed in~\cite{allen2025physics}, where a GLA model trained from scratch was able to catch up with GDN after adding short convolutions. For most benchmarks, short convolutions outperform LoRA. However, this may be due to the lower learning rate noted above, and we hypothesize that LoRA could benefit from longer training (e.g., 20M tokens, as used in prior work~\citep{lan2025liger, zhang2024lolcats}).

A risk of this setup is that the additional representational capacity introduced by short convolutions and LoRA increases the likelihood of forgetting or overfitting to the training data. This becomes evident when comparing MMLU accuracy, which is more knowledge-based, versus LAMBADA, which is context-based. MMLU accuracy benefits little from short convolutions or LoRA, and the accuracy gap between full and linear models remains around 6\%. For LAMBADA, by contrast, the gap shrinks, indicating that additional capacity helps the model handle context more effectively.

\begin{table*}[t]
\centering
\vskip -10pt
\caption{Downstream performance of the linearized models: LoRA or Short Convolution (SC), averaged over 3 random seeds. See App.~\ref{app:sc-lora-detailed} for detailed results.}
\vskip -5pt
\label{tab:downstream_sc_lora_swa}
\resizebox{\linewidth}{!}{
\begin{tabular}{clcccccccc}
\toprule
 && \multicolumn{4}{c}{Llama3.1 8b} & \multicolumn{4}{c}{Qwen3 8b} \\
\cmidrule(lr){3-6} \cmidrule(lr){7-10}
 && Param & CR & MMLU & Lambada & Param & CR & MMLU & Lambada \\
\midrule  \rowcolor{lightgray}
&Full Attention & --- & 73.98 & 65.25 & 74.79 & ---&  73.73 & 76.94 & 70.95 \\
\multirow{3}{*}{\rotatebox[origin=c]{90}{\small{SC}}}
&GLA      &  6.1M & 73.82 \stds{0.01} & 59.13 \stds{0.18} & 68.70 \stds{0.12} &6.9M & 73.52 \stds{0.02} & 70.58 \stds{0.17} & 64.20 \stds{0.05}  \\
&kGLA     & 10.3M &  73.80 \stds{0.00} & 59.02 \stds{0.12} & 68.65 \stds{0.11} & 11.6M &  73.51 \stds{0.03} & 70.75 \stds{0.17} & 64.20 \stds{0.06} \\
&GDN      & 10.3M & 73.77 \stds{0.02} & 59.17 \stds{0.17} & 69.01 \stds{0.03}  & 11.6M & 73.49 \stds{0.00} & 71.10 \stds{0.13} & 64.64 \stds{0.12} \\
\midrule
\multirow{3}{*}{\rotatebox[origin=c]{90}{\small{LoRA}}}
&GLA      & 11.1M & 71.71 \stds{0.09} & 50.10 \stds{0.59} & 67.51 \stds{1.16}  & 12.5M&  74.67 \stds{0.20} & 65.99 \stds{0.27} & 66.65 \stds{0.35}  \\
&kGLA     & 15.3M & 71.65 \stds{0.24} & 49.95 \stds{0.39} & 67.79 \stds{0.60}  & 17.3M &  74.97 \stds{0.08} & 66.29 \stds{0.30} & 65.48 \stds{0.51}\\
&GDN      & 15.3M & 72.80 \stds{0.25} & 54.80 \stds{2.79} & 70.29 \stds{1.96}  & 17.3M & 75.17 \stds{0.21} & 69.61 \stds{0.19} & 67.94 \stds{0.77}\\
\midrule
\multicolumn{3}{l}{\textbf{Gap with Full, \%}} & \textcolor{ForestGreen}{0.16} & \textcolor{red}{6.08} & \textcolor{ForestGreen}{4.50} & & \textcolor{ForestGreen}{-1.44} & \textcolor{red}{5.84} & \textcolor{ForestGreen}{3.01}\\
\midrule
\bottomrule
\end{tabular}
}
\vskip -10pt
\end{table*}

\section{Final Evaluation}

In the previous section, we established that GDN consistently outperforms other linear variants. While prior work shows SWA is essential for recovering performance, our analysis (Sec.\ref{sec:delta-vs-gating}) revealed that linear mechanisms fundamentally struggle to approximate the attention patterns of early sequence tokens. Finally, short convolutions are added to enhance representational capacity.
We now evaluate this final setup further, comparing it to other linearization methods, adding long context benchmarks, other models sizes and testing its performance in hybrid setting.
 
\paragraph{Common Reasoning.}
Table~\ref{tab:res_llama31} reports our final results and compares against prior linearization methods. Under the same 64-token sliding-window budget used by prior work, our linearized model matches or exceeds the best existing linearization baselines on the common reasoning suite, while requiring substantially fewer additional training tokens. In particular, we outperform LoLCaT~\citep{zhang2024lolcats} and Liger-GLA~\cite{lan2025liger} on 5-shot MMLU, yielding a noticeably smaller gap to the full-attention teacher. Finally, our approach is competitive with Llamba~\cite{bick2025llamba} on common reasoning despite Llamba relying on full-model finetuning over a much larger token budget (12B), highlighting the strength of delta-style updates plus minimal compensations in the post hoc setting.

Several recent linearization works further improve performance by augmenting the caching strategy beyond a fixed 64-token sliding window, typically by retaining additional globally important tokens. Lizard~\citep{van2025lizard} uses 128 cache tokens and 4 meta, while LoLA~\citep{mcdermott2025lola} and STILL~\citep{meng2026still} select additional 64 cache tokens with content-based criteria. These mechanisms are complementary to our approach: even without specialized selection, simply matching their total cache budget by enlarging the sliding window already yields stronger 5-shot MMLU (Table~\ref{tab:res_llama31}). We expect that combining GDN-based linearization with content-aware token selection could further improve performance at a fixed budget.

\begin{table*}[t]
\centering
\vskip -10pt
\caption{LLama3.1-8b Linearization results comparison with previous works.}
\vskip -5pt
\label{tab:res_llama31}
\resizebox{\linewidth}{!}{
\begin{tabular}{lcccllllll|l}
\toprule
Model &
Tok. \small{(B)} &
Cache Size &
Param. &
PIQA &
ARC-e &
ARC-c &
Hella. &
Wino. &
MMLU &
Avg. 

\\
\midrule \rowcolor{lightgray}
LLaMA 3.1 8B & - & $\infty$ & - & 79.05 & 82.15 & 54.78 & 79.34 & 74.59 & 65.25 & 72.52 \\
\multicolumn{11}{l}{\textbf{Linearized}} \\
Llamba  & 12  & 0 & 8B  & {80.9} & {82.5} & 54.6 & 77.6 & 73.3 &{60.0} & 71.5 \\ 
LolCaT   &0.04 & 64 &  & {80.96} & {82.37 }& 54.44 & {79.07} & 69.69 & 54.88 & 70.24 \\ 
Liger-GLA &  0.02& 64 &  & 80.5  & 81.8  & {55.6}  & 75.4  & 69.5  & 46.9  & 68.3 \\ 
\textsc{Ours (GDN)} & 0.01 & 64 & 10.3M  &79.07 \stds{0.02}	&81.97 \stds{0.01}&	54.92 \stds{0.08}&	78.79 \stds{0.07}&	74.11 \stds{0.00}&	59.17 \stds{0.17} & 71.34 \stds{0.05}\\
\midrule
\multicolumn{11}{l}{\textbf{Extended Cache budget}} \\
Lizard & 0.04 & 132 &  & 82& 83.5 &56.7& 79.3& 71.7& 61.2& 74.6\\ 
LoLa   & 0.04 & 128 &  & 81.6& 82.5& 55.4& 79.8& 73.6& 57.6 & 71.75 \\
STILL (concurrent)  & 0.04 & NA &  & {81.3} & {83.0} & {56.7} & {79.0} & {73.4} & {61.3} & {72.5} \\
\textsc{Ours (GDN)} & 0.01 & 128 & 10.3M  & 79.16 \stds{0.00} & 82.15 \stds{0.00} & 54.86 \stds{0.00} & 79.34 \stds{0.00} & 74.11 \stds{0.00} & 63.22 \stds{0.13} & 72.14 \stds{0.02} \\
\midrule
\bottomrule
\end{tabular}
}
\end{table*}

\paragraph{Long Context.} To evaluate long-context capabilities, we use S-NIAH and an extended RULER benchmarks (App.~\ref{tab:ruler_extended_4k}). 
Under the fair 128-window comparison, we outperform prior linearization baselines, and with a larger cache budget it becomes competitive with strong caching-based methods.

\paragraph{Scalability.} To assess the scalability of the linearization approach to models of various sizes, we apply it to the full line of dense Qwen3 models from 0.6B to 32B parameters. We use Base models, except for the largest 32B version, for which only an instruction-tuned variant is available. The results in Figure~\ref{fig:qwen_scaling} show that the linearized performance scales consistently with the original base model.

\paragraph{Hybrid Model.} We evaluate a hybrid setting where only a subset of layers retain full attention and the remaining layers are replaced with GDN (Fig.~\ref{fig:hybrid_res}). We select the retained full-attention layers either (i) uniformly across depth or (ii) using our hardness signals (key concentration and projection residual) to preferentially keep the hardest-to-linearize layers. We consider keeping 10\% - 50\% of layers; for LLaMA-3.1-8B (32 layers), this corresponds to retaining 3 - 16 full attention layers. When only a small fraction of layers is kept (10, 20\%), the metrics-based selection improves Lambada accuracy compared to uniform selection (Fig.~\ref{fig:hybrid_res}), indicating that preserving a few hard layers is particularly beneficial for context-dependent evaluation.

\begin{table}[t]
    \centering
    \vskip -5pt
    \begin{tabular}{ccc}
        \hspace{-12pt}
        \begin{minipage}[t]{0.59\textwidth}
            \hspace{-15pt} 
            \centering
            \includegraphics[width=1\linewidth]{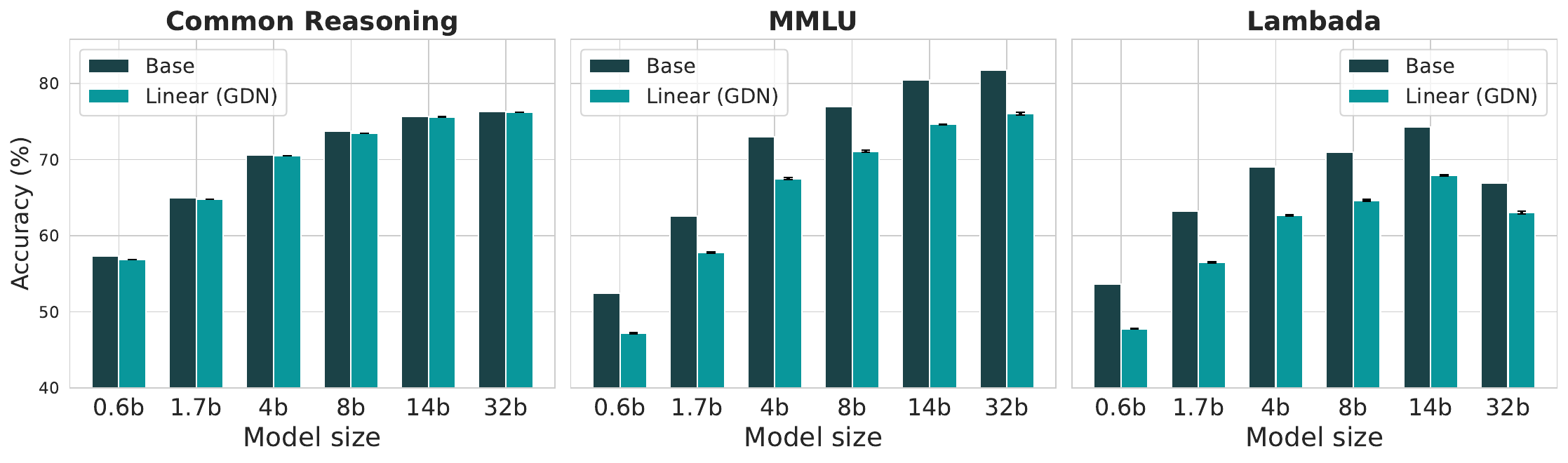} 
            \captionof{figure}{Downstream accuracies for linearized Qwen3 models of various sizes. See App.~\ref{app:qwen-scaling-detailed} for detailed results.} \label{fig:qwen_scaling}
        \end{minipage}&
        \begin{minipage}[t]{0.4\textwidth}
            \hspace{-10pt} 
            \centering
             \includegraphics[width=1\linewidth]{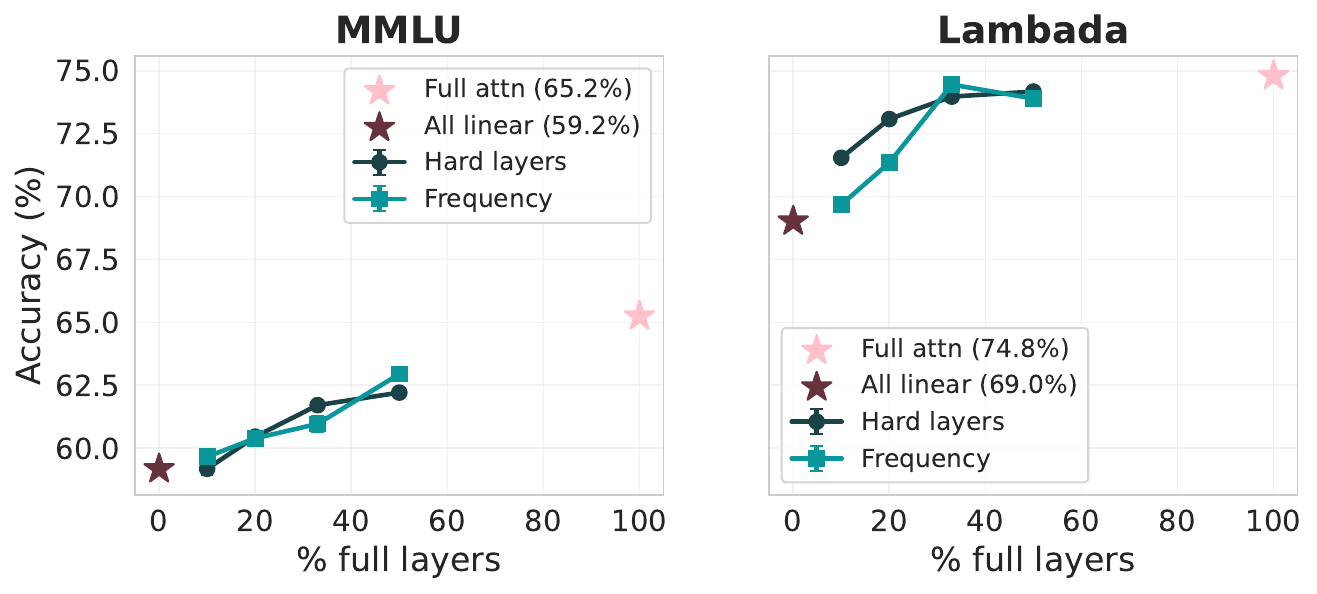} 
            \captionof{figure}{Hybrid setting: keeping full attention for 10\% to 50\% layers .}\label{fig:hybrid_res}
        \end{minipage} &
    \end{tabular}
    \vskip -15pt
\end{table}

\section{Conclusion and Limitations}
We study post hoc transformer linearization starting from a strict regime, and progressively relax it to quantify the effect of additional components (SWA, a small “sink-token” budget, and Short Convolutions). In the strict setting, delta-style updates are the most performant replacement, aligning with our first-order view that softmax induces key-dependent rank-1 corrections. With the final recipe, we match full attention model on common reasoning benchmarks, outperform prior linearization baselines on MMLU, and match the long-context performance of complex adaptive-caching frameworks.
\textbf{Limitations:}  Our first-order theory relies on approximate assumptions that may not hold uniformly across layers or very short prefixes. We leave the extension of adaptive caching strategies to GDN linearization to future works.

\newpage
\bibliographystyle{plain}
\bibliography{biblio}
\newpage
\appendix
\section{Proofs}\label{app:proofs}

\subsection{Linear Approximation}\label{app:proof1}
\begin{lemma}[First--order Taylor approximation of softmax]
Let 
$p(x) = \mathrm{softmax}(x)$,
$x_i = q_t^\top k_i$, and 
$\bar{x} = \frac{1}{t}\sum_{j=1}^t x_j .$
Then for each index $i$,
\[
p_i(x)
=
\frac{1}{t}
+
\frac{1}{t}\bigl(x_i - \bar{x}\bigr)
+
\varepsilon_i^{(1)}(x),
\]
where the error term satisfies
\[
\|\varepsilon^{(1)}(x)\|_\infty \le \frac{1}{2}\|x\|_2^2 .
\]
\end{lemma}

\begin{proof}
The derivative of softmax is given by the Jacobian
\[
J_{ij}(x) = p_i(x)(\delta_{ij}-p_j(x)).
\]
We evaluate softmax at the zero vector $x=0$.
At this point all entries are equal, so
\[
p_i(0)=\frac{1}{t}.
\]
Substituting into the Jacobian gives
\[
J_{ij}(0)=\frac{1}{t}\left(\delta_{ij}-\frac{1}{t}\right).
\]

We apply the first--order Taylor expansion around $x=0$:
\[
p(x)=p(0)+J(0)x+R(x),
\]
where $R(x)$ collects higher--order terms.

Since $p(0)=\frac{1}{t}\mathbf{1}$ and
\[
J(0)x=\frac{1}{t}\left(x-\bar{x}\mathbf{1}\right),
\]
we obtain
\[
p(x)
=
\frac{1}{t}\mathbf{1}
+
\frac{1}{t}\left(x-\bar{x}\mathbf{1}\right)
+
R(x).
\]

Because softmax is smooth, the remainder is bounded by
\[
\|R(x)\|_\infty \le \frac{1}{2}\|x\|_2^2 .
\]
\end{proof}

Substituting $x_i = q_t^\top k_i$ gives
\[
p_i(x)
\approx
\frac{1}{t}
+
\frac{1}{t}q_t^\top (k_i-\bar{k}).
\]

\subsection{Projection}\label{app:proof2}
\begin{lemma}[Projection approximation of centering]
Define the mean key and the projection matrix
\[
\bar{k}=\frac{1}{t}\sum_{j=1}^t k_j,
\qquad
\Pi_{\bar{k}}=\frac{\bar{k}\bar{k}^\top}{\|\bar{k}\|^2}.
\]
Assume all keys are normalized  $\|k_i\|=1 \quad \text{for all } i$. Then
\[
k_i-\bar{k}=(I-\Pi_{\bar{k}})k_i+r_i,
\]
where the residual $r_i$ satisfies
\[
\|r_i\|
\le
|\cos(\alpha_i)-1| + |1-\|\bar{k}\||,
\]
and $\alpha_i$ is the angle between $k_i$ and $\bar{k}$.
\end{lemma}

\begin{proof}
The matrix $\Pi_{\bar{k}}$ projects any vector onto the direction of
$\bar{k}$.
Applying it to $k_i$ gives
\[
\Pi_{\bar{k}}k_i
=
\bar{k}\frac{\langle k_i,\bar{k}\rangle}{\|\bar{k}\|^2}.
\]

We define the residual as the difference between exact centering and projection:
\[
r_i
=
(k_i-\bar{k})-(I-\Pi_{\bar{k}})k_i
=
\Pi_{\bar{k}}k_i-\bar{k}.
\]

Substituting the projection formula,
\[
r_i
=
\bar{k}\left(
\frac{\langle k_i,\bar{k}\rangle}{\|\bar{k}\|^2}-1
\right).
\]

Because $k_i$ is a unit vector, the inner product can be written as
\[
\langle k_i,\bar{k}\rangle=\|\bar{k}\|\cos(\alpha_i).
\]
Thus
\[
r_i=\bar{k}\left(\frac{\cos(\alpha_i)}{\|\bar{k}\|}-1\right).
\]

Taking norms yields
\[
\|r_i\|
=
\bigl|\cos(\alpha_i)-\|\bar{k}\|\bigr|.
\]
Finally, we rewrite the difference as a sum:
\[
\cos(\alpha_i)-\|\bar{k}\|
=
(\cos(\alpha_i)-1) + (1-\|\bar{k}\|).
\]
Taking absolute values and using the triangular inequality
\[
|a+b| \le |a| + |b|,
\]
which holds for all real numbers $a,b$, we obtain
\[
\|r_i\|
\le
|\cos(\alpha_i)-1| + |1-\|\bar{k}\||.
\]
\end{proof}

\subsection{Rank-1 approximation}\label{app:proof3}

\begin{lemma}[Rank--1 product approximation]
Let $\|k_j\|=1$ and define
\[
D_j = I - \beta_j k_j k_j^\top,
\qquad
\beta_j \in (0,1),
\qquad
M_t = \prod_{j=1}^t D_j .
\]

Then
\[
M_t
=
I
-
\sum_{j=1}^t \beta_j k_j k_j^\top
+
E_t,
\qquad
\|E_t\|_2 \le C\, t (\max_j \beta_j)^2.
\]

If the vectors $k_j$ are concentrated around their mean
$\bar{k}=\frac{1}{t}\sum_{j=1}^t k_j$, then
\[
M_t \approx I - c\,\Pi_{\bar{k}},
\quad
\Pi_{\bar{k}}=\frac{\bar{k}\bar{k}^\top}{\|\bar{k}\|^2},
\]
for some scalar $c \ge 0$.
\end{lemma}

\begin{proof}
We expand
\[
M_t
=
\prod_{j=1}^t (I-\beta_j k_j k_j^\top).
\]

Multiplying out yields
\[
\begin{aligned}
M_t
&=
I
-
\sum_{j=1}^t \beta_j k_j k_j^\top \\
&\quad
+
\sum_{1 \le j < \ell \le t}
\beta_j\beta_\ell
(k_j k_j^\top)(k_\ell k_\ell^\top)
+ \cdots .
\end{aligned}
\]

All terms of order two or higher are collected in $E_t$.

Each matrix $k_j k_j^\top$ satisfies
\[
\|k_j k_j^\top\|_2 \le \|k_j\|^2 = 1.
\]
Using
\[
\|AB\|_2 \le \|A\|_2\|B\|_2,
\]
every order $m\ge2$ term is bounded by $(\max_j\beta_j)^m$.
There are at most $O(t^m)$ such terms, hence
\[
\|E_t\|_2
\le
C \sum_{m=2}^\infty t(\max_j\beta_j)^m
\le
C\,t(\max_j\beta_j)^2.
\]

If the vectors $k_j$ are close in direction to their mean $\bar{k}$, then
\[
k_j k_j^\top \approx \cos^2(\alpha_j)\,\bar{k}\bar{k}^\top,
\]
where $\alpha_j$ is the angle between $k_j$ and $\bar{k}$.
Thus,
\[
\sum_{j=1}^t \beta_j k_j k_j^\top
\approx
\left(\sum_{j=1}^t \beta_j \cos^2(\alpha_j)\right)
\bar{k}\bar{k}^\top.
\]

Defining
\[
c = \sum_{j=1}^t \beta_j \cos^2(\alpha_j),
\]
with $c \ge 0$, we obtain
\[
M_t \approx I - c\,\Pi_{\bar{k}} .
\]
\end{proof}

\newpage
\section{Detailed Results}\label{app:detailed_red}

Here we report detailed accuracies for all the benchmarks and additional experimental evaluations.

\subsection{Pure Linear Attention}\label{app:pure-linear-detailed}

\begin{table*}[h]
\centering
\caption{Downstream performance of the linearized models without SWA and without LoRA.  We report the accuracy on common reasoning benchmarks (PiQA, ARC-e, ARC-c, HellaSwag, WinoGrande), 5-shot MMLU and Lambada (OpenAI). For ARC-c and HellaSwag we use normalized accuracy. Results are averaged over 3 random seeds.}
\label{tab:downstream_pure_detailed}
\resizebox{\linewidth}{!}{
\begin{tabular}{lcccccccc|c}
\hline
Model &
Params. &
PIQA &
ARC-e &
ARC-c &
Hella. &
Wino. &
MMLU \tiny{(5-shot)} &
Lambada &
Avg.
\\
\hline \rowcolor{lightgray}
Llama 3.1 8b  & -- &79.05 & 82.15 & 54.78 & 79.34 & 74.59 & 65.25 & 74.79  & 72.84 \\
GLA & 4.3M &   76.02 \stds{0.17} & 69.28 \stds{0.34} & 38.91 \stds{0.18} & 67.57 \stds{0.09} & 57.27 \stds{0.13} & 27.13 \stds{0.47} & 28.26 \stds{0.11} & 52.06 \stds{0.03} \\
kGLA & 8.5M & 75.81 \stds{0.02} & 71.06 \stds{0.17} & 39.79 \stds{0.41} & 67.92 \stds{0.01} & 58.41 \stds{0.44} & 27.37 \stds{0.09} & 27.49 \stds{0.31} & 52.55 \stds{0.04} \\
GDN & 8.5M &  76.30 \stds{0.34} & 70.81 \stds{0.47} & 39.36 \stds{0.27} & 70.75 \stds{0.08} & 58.17 \stds{0.34} & 27.55 \stds{0.36} & 44.67 \stds{0.26} & 55.37 \stds{0.11} \\
Hedgehog & 10.6M & 51.83 \stds{0.05} & 26.28 \stds{0.07} & 25.97 \stds{0.27} & 25.39 \stds{0.03} & 51.17 \stds{0.19} & 24.95 \stds{0.03} & 0.00 \stds{0.00} & 29.37 \stds{0.01} \\
\midrule  \rowcolor{lightgray}
Qwen3 8b & - &79.16 & 81.73 & 56.14 & 78.69 & 72.93 & 76.94 & 70.95 & 73.79\\
GLA & 4.7M & 67.59 \stds{0.31} & 58.02 \stds{0.05} & 30.12 \stds{0.38} & 52.83 \stds{0.19} & 51.64 \stds{0.37} & 24.50 \stds{0.07} & 11.29 \stds{0.07} & 42.29 \stds{0.01}  \\
kGLA & 9.6M & 67.23 \stds{0.07} & 59.61 \stds{0.31} & 31.06 \stds{0.05} & 53.32 \stds{0.06} & 52.57 \stds{0.09} & 25.23 \stds{0.23} & 11.33 \stds{0.19} & 42.91 \stds{0.08}   \\
GDN & 9.6M & 68.53 \stds{0.29} & 64.79 \stds{0.15} & 35.52 \stds{0.38} & 59.64 \stds{0.19} & 54.28 \stds{0.23} & 25.57 \stds{0.23} & 23.60 \stds{0.33} & 47.42 \stds{0.08}  \\
Hedgehog & 11.9M &  52.79 \stds{0.04} & 26.52 \stds{0.02} & 24.63 \stds{0.08} & 25.88 \stds{0.03} & 49.67 \stds{0.27} & 22.91 \stds{0.01} & 0.00 \stds{0.00} & 28.91 \stds{0.04}  \\
\hline
\end{tabular}
}
\end{table*}

\paragraph{Latency and memory}
All latecy measurements are run on a single NVIDIA H100 GPU with batch size 8. For each prompt length on the x-axis, we (i) run a forward pass on the prompt and then (ii) decode and time the generation of 16 new tokens. We report end-to-end latency for these 16 tokens, and peak GPU memory allocated during the run. We compare the linear attention variants against the full-attention baseline implemented with \texttt{FlashAttention-2 (FA2)}.

\vspace{2mm}
\paragraph{Approximation error}
We compute approximation error as the mean squared error (MSE) between the attention-block outputs of the frozen full-attention model and the outputs produced by the corresponding linearized replacement, evaluated on the same inputs. We use a random subset of 8,000 sequences from the Clean Alpaca~\citep{alpaca} dataset to compute this metric (this can serve as a hold-out set, as we use DCLM-Edu to train the model). For each layer $\ell$, we compute
\[
\mathrm{MSE}^{(\ell)} = \frac{1}{Td}\left\|y^{(\ell)}_{\text{full}} - y^{(\ell)}_{\text{lin}}\right\|_2^2,
\quad
\mathrm{nMSE}^{(\ell)} = \frac{\mathrm{MSE}^{(\ell)}}{\mathrm{Var}\!\left(y^{(\ell)}_{\text{full}}\right)},
\]
where $y^{(\ell)}_{\text{full}}, y^{(\ell)}_{\text{lin}} \in \mathbb{R}^{T\times d}$ denote the attention-block outputs for a sequence of length $T$, and $\mathrm{Var}(\cdot)$ is the empirical variance of the full-attention activations at that layer. 

\begin{figure}[h] %
\centering
\begin{subfigure}{0.31\linewidth}
  \centering
    \includegraphics[width=1\linewidth]{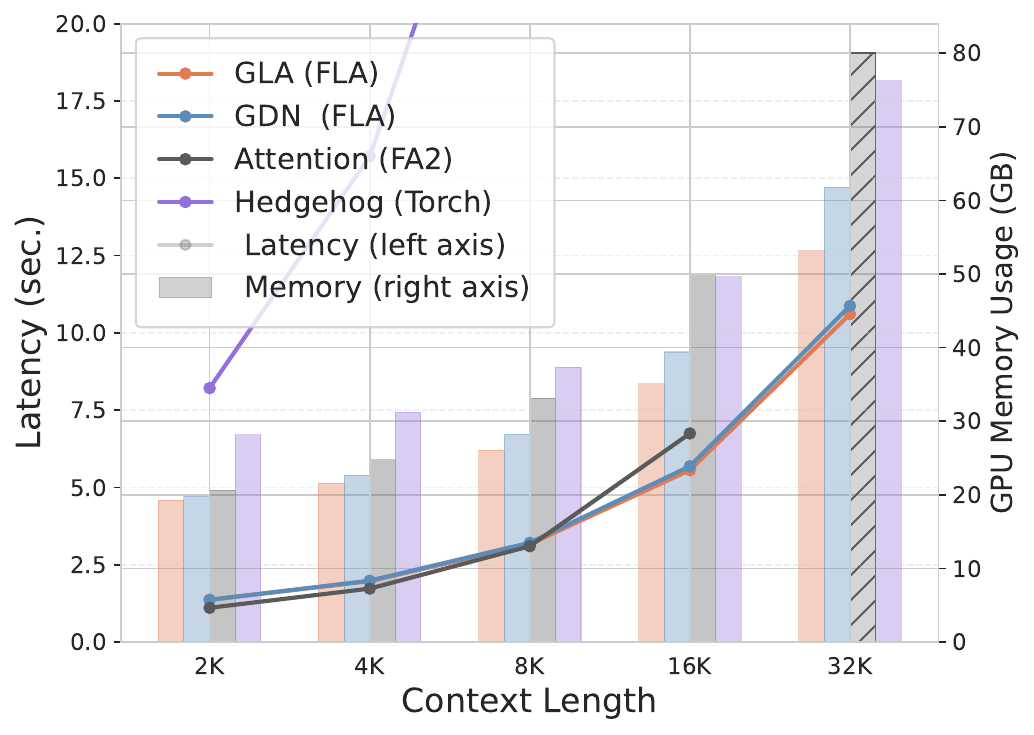} 
  \caption{Latency}
\end{subfigure}\hfill
\begin{subfigure}{0.31\linewidth}
  \centering
  \includegraphics[width=1\linewidth]{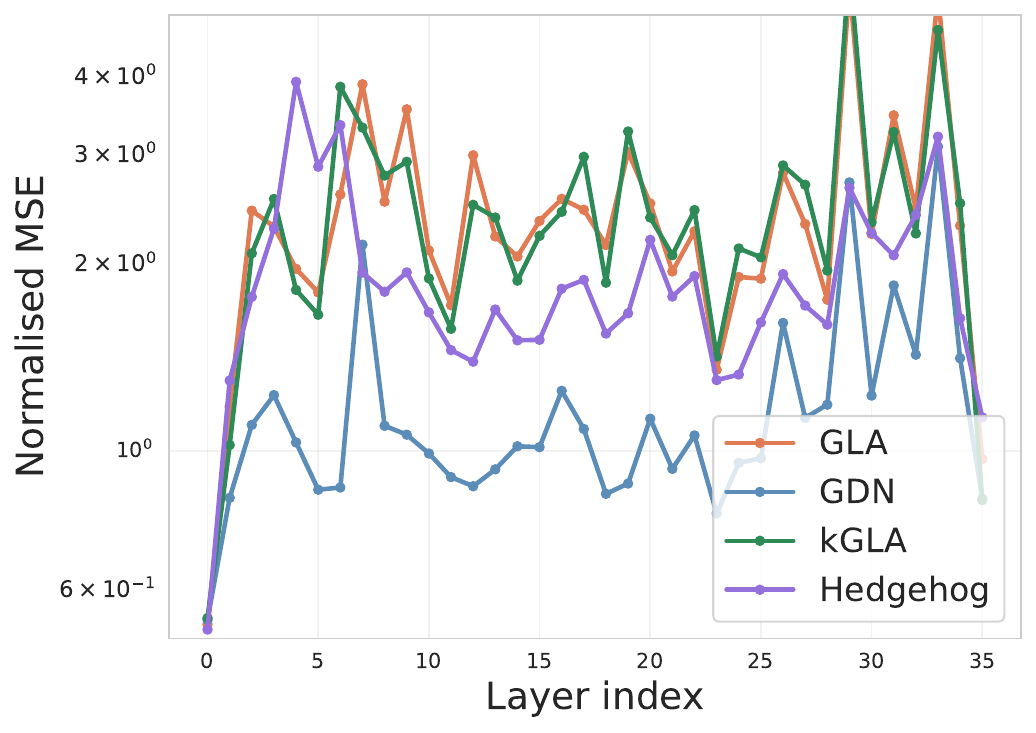} 
  \caption{MSE Per Layer}
\end{subfigure}\hfill
\begin{subfigure}{0.31\linewidth}
  \centering
  \includegraphics[width=1\linewidth]{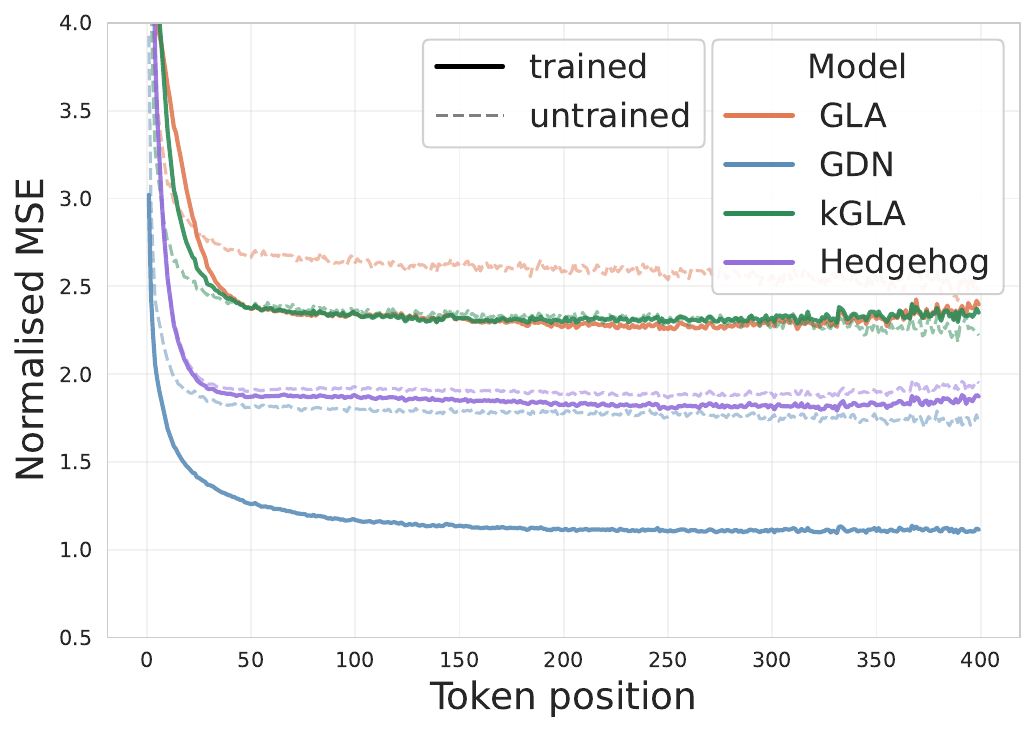} 
  \caption{MSE per Token}
\end{subfigure}
\caption{Comparison of linear attention variants for  Qwen3-8b linearization.}\label{fig:pure_mse_latecy_qwen}
\vskip -15pt
\end{figure}

\paragraph{Assumption Checks: Projection Residuals and Key Concentration}
\label{app:assumption_checks}

Our first-order analysis relies on two empirically checkable conditions: 
\begin{itemize}
    \item keys are \emph{approximately concentrated} around a common direction, and 
    \item subtracting the mean key is well-approximated by projecting orthogonally to that mean.
\end{itemize}
We evaluate these conditions via two diagnostics computed from the frozen full-attention model keys.

\noindent\textbf{Notation.}
Let $k^{(\ell,h)}_i \in \mathbb{R}^{d}$ denote the (RoPE-processed) key at layer $\ell$, head $h$, and token position $i$.
Because Lem.~2--3 assume unit-norm keys, we use normalized keys
\[
\tilde{k}^{(\ell,h)}_i \;=\; \frac{k^{(\ell,h)}_i}{\|k^{(\ell,h)}_i\|_2 + \varepsilon},
\]
with a small $\varepsilon>0$ for numerical stability.
For a prefix of length $t$, define the running mean key direction
\[
\bar{k}^{(\ell,h)}_t \;=\; \frac{1}{t}\sum_{j=1}^{t}\tilde{k}^{(\ell,h)}_j.
\]

\paragraph{Metric 1: Key Concentration}
We measure \emph{key concentration} at token position $t$ as the cosine similarity between the key and its running mean direction:
\[
C^{(\ell,h)}_t \;=\; \cos\bigl(\tilde{k}^{(\ell,h)}_t,\bar{k}^{(\ell,h)}_t\bigr)
\;\in\;[-1,1].
\]
When keys are strongly aligned , $\bar{k}^{(\ell,h)}_j$ becomes a stable direction and $C_t$ stays close to 1; when keys are dispersed, and $C_t$ decreases. This directly corresponds to the small angle to the mean regime used in Lemma~2--3 (via $\cos(\alpha_i)$).

\paragraph{Metric 2: Projection residual.}
Lemma~2 replaces subtracting the mean key $(\tilde{k}_i-\bar{k}_t)$ with a projection orthogonal to the mean direction $(I-\Pi_{\bar{k}_t})\tilde{k}_i$, where
\[
\Pi_{\bar{k}_t} \;=\; \frac{\bar{k}_t\bar{k}_t^\top}{\|\bar{k}_t\|_2^2 + \varepsilon}.
\]
We quantify the quality of this approximation using the \emph{projection residual}:
\[
R^{(\ell,h)}_t \;=\; \frac{1}{t}\sum_{i=1}^{t}
\left\|\,(\tilde{k}^{(\ell,h)}_i - \bar{k}^{(\ell,h)}_t) - (I-\Pi_{\bar{k}^{(\ell,h)}_t})\tilde{k}^{(\ell,h)}_i \right\|_2.
\]
Equivalently, the inner term simplifies to 
$\Pi_{\bar{k}_t}\tilde{k}_i-\bar{k}_t$, so the residual measures how well 
\emph{mean subtraction} is captured by an \emph{orthogonal projection} onto the complement of the mean-key direction. 
Smaller $R_t$ indicates Lem.~2 is a good approximation for that prefix.

For each evaluation sequence, we compute $C^{(\ell,h)}_t$ and $R^{(\ell,h)}_t$ for all token indices $t$ and all (layer, head) pairs, then average across sequences and across heads/layers:
\[
C_t = \mathbb{E}_{\text{seq},\ell,h}\bigl[C^{(\ell,h)}_t\bigr], \qquad
R_t = \mathbb{E}_{\text{seq},\ell,h}\bigl[R^{(\ell,h)}_t\bigr].
\]
These produce the token-index curves shown in Fig.~\ref{fig:lemmas}, \ref{fig:lemmas_qwen}(b). 

To relate the assumptions to approximation quality, we additionally bin token positions by $C^{(\ell,h)}_t$ (key concentration bins) or $R^{(\ell,h)}_t$ (projection residual bins) and compute the corresponding normalized MSE within each bin, yielding the binned trends in Fig.~\ref{fig:lemmas}, \ref{fig:lemmas_qwen}(a).

\begin{figure}[h] %
\centering
\captionsetup{type=figure}

\begin{subfigure}{0.66\linewidth}
  \centering
  \includegraphics[width=1\linewidth]{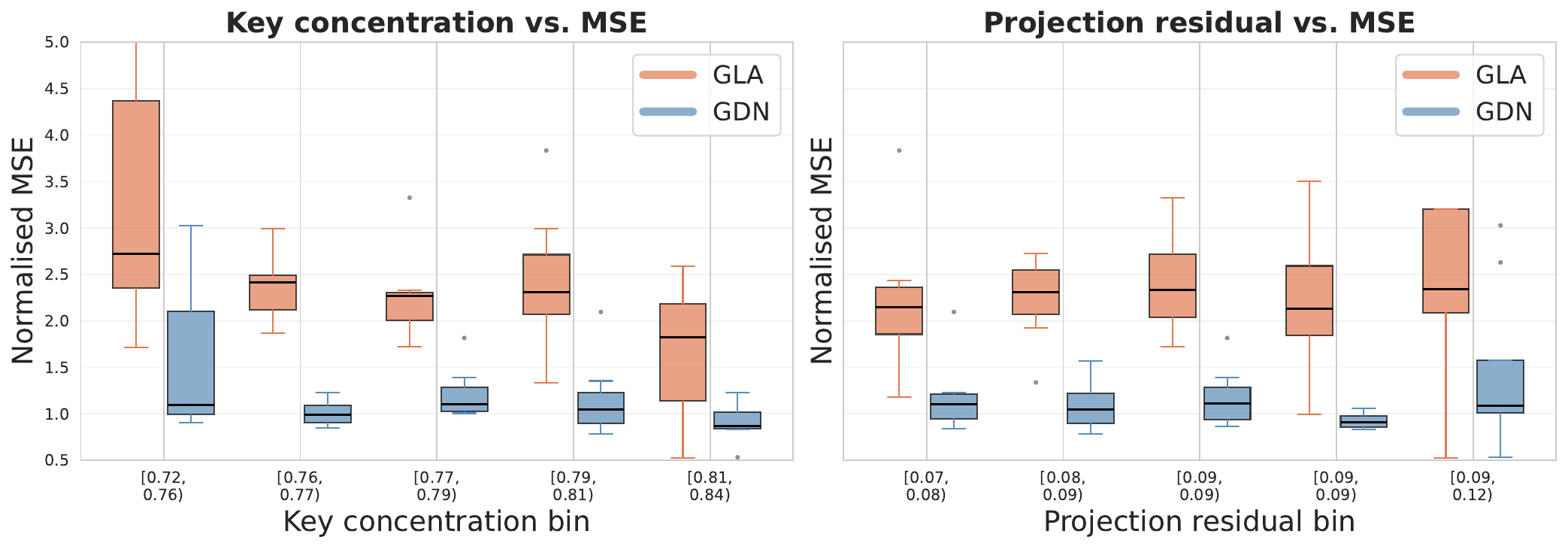} 
            \caption{Per layer normalized MSE}
\end{subfigure}\hfill
\begin{subfigure}{0.33\linewidth}
  \centering
\includegraphics[width=1\linewidth]{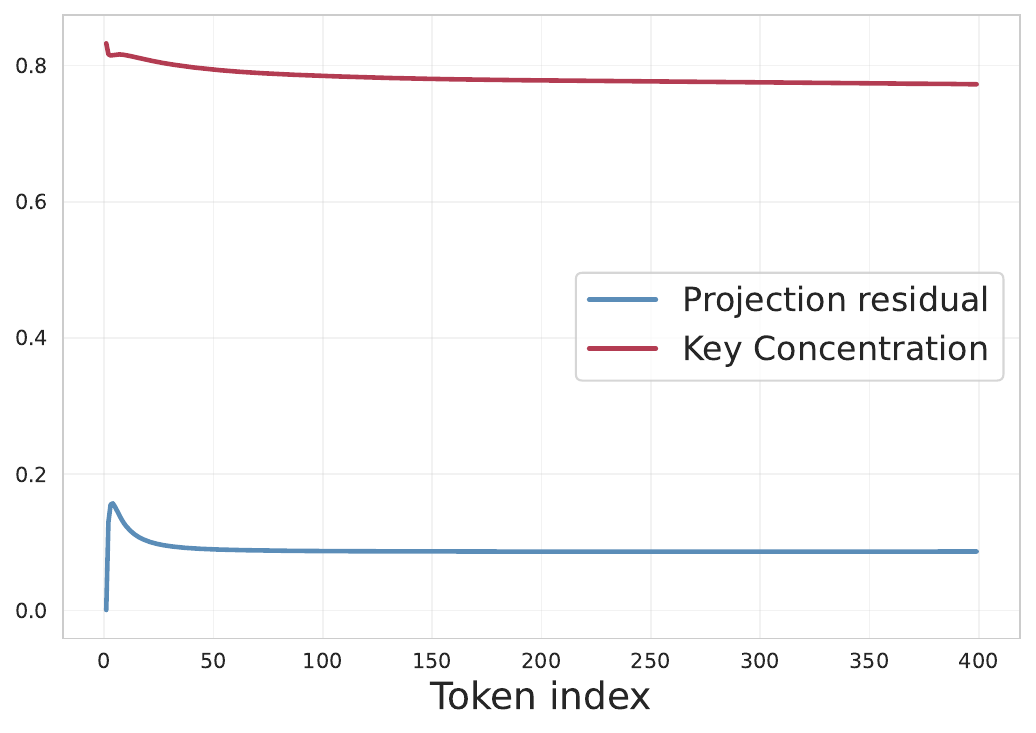} 
            \caption{Token index}
\end{subfigure}

\caption{Key Concentration and Projection Residual for Qwen3-8b}\label{fig:lemmas_qwen}
\end{figure}

\newpage
\subsection{SWA and Sink Tokens}\label{app:swa-sink-detailed}

\begin{table*}[h]
\centering
\caption{Downstream performance of the linearized models with SWA and Sink Tokens. We report the accuracy on common reasoning benchmarks (PiQA, ARC-e, ARC-c, HellaSwag, WinoGrande), 5-shot MMLU and Lambada (OpenAI). For ARC-c and HellaSwag we use normalized accuracy. Results are averaged over 3 random seeds.}
\label{tab:downstream_pure_swa_detailed}
\resizebox{\linewidth}{!}{
\begin{tabular}{lcccccccc|c}
\hline
Model &
Params. &
PIQA &
ARC-e &
ARC-c &
Hella. &
Wino. &
MMLU \tiny{(5-shot)} &
Lambada &
Avg.
\\
\hline \rowcolor{lightgray}
Llama 3.1 8b  & -- &79.05 & 82.15 & 54.78 & 79.34 & 74.59 & 65.25 & 74.79  & 72.84 \\
\multicolumn{2}{l}{\textbf{SWA (56) + Sink (8)}} \\
GLA & 4.3M & 78.94 \stds{0.03} & 81.94 \stds{0.04} & 54.81 \stds{0.08} & 75.55 \stds{0.03} & 74.11 \stds{0.00} & 44.25 \stds{0.08} &  59.31 \stds{0.11} & 66.99 \stds{0.02}  \\
kGLA & 8.5M & 78.98 \stds{0.04} & 81.93 \stds{0.04} & 54.58 \stds{0.03} & 75.63 \stds{0.04} & 74.11 \stds{0.00} & 43.28 \stds{0.03} &  59.57 \stds{0.19} & 66.87 \stds{0.02}\\
GDN & 8.5M &  79.07 \stds{0.04} & 82.10 \stds{0.01} & 54.92 \stds{0.03} & 78.86 \stds{0.03} & 74.11 \stds{0.00} & 58.88 \stds{0.14} & 68.82 \stds{0.13} & 70.97 \stds{0.03} \\
\textbf{SWA (64)} \\
GLA & 4.3M &    77.64 \stds{0.03} & 81.12 \stds{0.04} & 53.36 \stds{0.08} & 55.92 \stds{0.22} & 74.59 \stds{0.00} & 39.50 \stds{0.60} &  36.19 \stds{0.14} & 59.76 \stds{0.10} \\
kGLA & 8.5M &  77.53 \stds{0.05} & 81.00 \stds{0.03} & 53.24 \stds{0.05} & 55.44 \stds{0.10} & 74.59 \stds{0.00} & 37.63 \stds{0.02} & 36.86 \stds{0.16} & 59.47 \stds{0.04} \\
GDN & 8.5M &  78.06 \stds{0.07} & 80.98 \stds{0.05} & 53.58 \stds{0.15} & 56.44 \stds{0.21} & 74.59 \stds{0.00} & 51.57 \stds{0.77} & 41.92 \stds{0.22} & 62.45 \stds{0.09} \\
\midrule  \rowcolor{lightgray}
Qwen3 8b & - &79.16 & 81.73 & 56.14 & 78.69 & 72.93 & 76.94 & 70.95 & 73.79\\
\multicolumn{2}{l}{\textbf{SWA (56) + Sink (8)}} \\
GLA & 4.9M & 79.16 \stds{0.03} & 82.07 \stds{0.04} & 56.34 \stds{0.08} & 75.94 \stds{0.01} & 71.98 \stds{0.00} & 67.40 \stds{0.05} & 56.50 \stds{0.29} & 69.91 \stds{0.06}  \\
kGLA & 9.6M &79.14 \stds{0.02} & 81.99 \stds{0.02} & 56.14 \stds{0.00} & 76.15 \stds{0.05} & 71.98 \stds{0.00} & 67.70 \stds{0.17} & 57.05 \stds{0.12} & 70.02 \stds{0.03}  \\
GDN & 9.6M &79.22 \stds{0.05} & 81.82 \stds{0.02} & 56.37 \stds{0.03} & 77.90 \stds{0.03} & 71.98 \stds{0.00} & 70.97 \stds{0.18} & 64.53 \stds{0.04} & 71.83 \stds{0.03}  \\
\textbf{SWA (64)} \\
GLA & 4.9M &78.75 \stds{0.02} & 81.09 \stds{0.04} & 55.66 \stds{0.12} & 66.99 \stds{0.08} & 72.93 \stds{0.00} & 48.10 \stds{0.51} &  52.55 \stds{0.21} & 65.15 \stds{0.07} \\
kGLA & 9.6M & 78.87 \stds{0.13} & 81.20 \stds{0.01} & 55.55 \stds{0.00} & 67.10 \stds{0.16} & 72.93 \stds{0.00} & 48.86 \stds{0.53} &  53.30 \stds{0.22} & 65.40 \stds{0.13} \\
GDN & 9.6M & 78.91 \stds{0.02} & 81.48 \stds{0.06} & 55.92 \stds{0.15} & 67.95 \stds{0.18} & 72.93 \stds{0.00} & 53.70 \stds{0.20} & 59.30 \stds{0.13} & 67.17 \stds{0.07}\\
\hline
\end{tabular}
}
\end{table*}

\newpage
\subsection{Short Convolution and LoRA}\label{app:sc-lora-detailed}

\begin{table*}[h]
\centering
\caption{Downstream performance of the linearized models with SWA and modified projections. We report the accuracy on common reasoning benchmarks (PiQA, ARC-e, ARC-c, HellaSwag, WinoGrande), 5-shot MMLU and Lambada (OpenAI). For ARC-c and HellaSwag we use normalized accuracy. Results are averaged over 3 random seeds.}
\label{tab:downstream_sc_lora_swa_detailed}
\resizebox{\linewidth}{!}{
\begin{tabular}{lcccccccc|c}
\hline
Model &
Params. &
PIQA &
ARC-e &
ARC-c &
Hella. &
Wino. &
MMLU \tiny{(5-shot)} &
Lambada &
Avg.
\\
\hline \rowcolor{lightgray}
Llama 3.1 8b  & -- &79.05 & 82.15 & 54.78 & 79.34 & 74.59 & 65.25 & 74.79  & 72.84 \\
\textbf{Short Conv} \\
GLA & 6.1M &79.11 \stds{0.00} & 81.99 \stds{0.00} & 54.95 \stds{0.00} & 78.93 \stds{0.03} & 74.11 \stds{0.00} & 59.13 \stds{0.18} & 68.70 \stds{0.12} & 70.99 \stds{0.04} \\
kGLA & 10.3M & 79.11 \stds{0.00} & 81.99 \stds{0.02} & 54.92 \stds{0.03} & 78.87 \stds{0.02} & 74.11 \stds{0.00} & 59.02 \stds{0.12} & 68.65 \stds{0.11} & 70.95 \stds{0.03}  \\
GDN & 10.3M & 79.07 \stds{0.02} & 81.97 \stds{0.01} & 54.92 \stds{0.08} & 78.79 \stds{0.07} & 74.11 \stds{0.00} & 59.17 \stds{0.17} & 69.01 \stds{0.03} & 71.01 \stds{0.04} \\
\textbf{LoRA} \\
GLA &11.1M & 78.87 \stds{0.16} & 79.55 \stds{0.14} & 51.25 \stds{0.32} & 76.93 \stds{0.04} & 71.93 \stds{0.28} & 50.10 \stds{0.59} & 67.51 \stds{1.16} & 68.02 \stds{0.31}\\
kGLA &15.3M & 79.00 \stds{0.14} & 79.92 \stds{0.06} & 51.62 \stds{0.68} & 76.84 \stds{0.21} & 70.88 \stds{0.30} & 49.95 \stds{0.39} & 67.79 \stds{0.60} & 68.00 \stds{0.29} \\
GDN &15.3M &  79.47 \stds{0.29} & 80.86 \stds{0.34} & 53.33 \stds{0.58} & 78.03 \stds{0.31} & 72.32 \stds{0.09} & 54.80 \stds{2.79} & 70.29 \stds{1.96} & 69.87 \stds{0.85}\\
\midrule  \rowcolor{lightgray}
Qwen3 8b & - &79.16 & 81.73 & 56.14 & 78.69 & 72.93 & 76.94 & 70.95 & 73.79\\
\textbf{Short Conv} \\
GLA & 6.9M &79.33 \stds{0.03} & 81.92 \stds{0.04} & 56.43 \stds{0.12} & 77.93 \stds{0.07} & 71.98 \stds{0.00} & 70.58 \stds{0.17} & 64.20 \stds{0.05} & 71.77 \stds{0.04} \ \\
kGLA& 11.6M &79.34 \stds{0.02} & 81.89 \stds{0.04} & 56.43 \stds{0.03} & 77.93 \stds{0.08} & 71.98 \stds{0.00} & 70.75 \stds{0.17} & 64.20 \stds{0.05} & 71.77 \stds{0.04} \ \\ 
GDN & 11.6M & 79.29 \stds{0.04} & 81.86 \stds{0.02} & 56.48 \stds{0.00} & 77.84 \stds{0.03} & 71.98 \stds{0.00} & 71.10 \stds{0.13} & 64.64 \stds{0.12} & 71.88 \stds{0.02} \\
\textbf{LoRA} \\
GLA & 12.5M & 79.18 \stds{0.21} & 84.47 \stds{0.06} & 62.17 \stds{0.59} & 75.56 \stds{0.04} & 71.95 \stds{0.17} & 65.99 \stds{0.27} & 66.65 \stds{0.35} & 72.28 \stds{0.15}\\
kGLA & 17.4M&  78.91 \stds{0.17} & 84.85 \stds{0.11} & 61.83 \stds{0.27} & 76.04 \stds{0.02} & 73.22 \stds{0.15} & 66.29 \stds{0.30} & 65.48 \stds{0.51} & 72.37 \stds{0.14}\\
GDN & 17.3M &  79.27 \stds{0.17} & 84.39 \stds{0.15} & 62.29 \stds{0.39} & 76.56 \stds{0.04} & 73.35 \stds{0.50} & 69.61 \stds{0.19} & 67.94 \stds{0.77} & 73.34 \stds{0.05} \\
\hline
\end{tabular}
}
\end{table*}

\newpage
\subsection{Long context Evaluation on Ruler Benchmark}\label{app:ruler_results}
We evaluate long-context behavior on S-NIAH and an extended RULER benchmark at 4k context length (Tables~\ref{tab:ruler_niah} and \ref{tab:ruler_extended_4k}). Under the fairest comparison setting, prior linearization baselines that use the same 64-token local windowm, our GDN-based linearization substantially outperforms LoLCATs and Liger-GLA by a wide margin, consistently across both common reasoning (Table~\ref{tab:res_llama31}) and long-context evaluation (Table~\ref{tab:ruler_extended_4k}). 

At the same time, Table~\ref{tab:ruler_extended_4k} shows a strong dependence on cache budget. 
The small-cache configuration (128 total tokens: 8 sink + 120 sliding) is not viable for retrieval-heavy subsets. 
Increasing the cache (e.g., 896 tokens) makes GDN competitive on average, but exposes a clear profile: strong on aggregation/QA-style tasks and weaker on tasks requiring multiple independent memory items in parallel (e.g., multi-query/multi-value retrieval and variable tracking).
Since we use a fixed-budget cache (8 sink tokens plus a sliding window) and do not perform adaptive token selection at inference time, these results suggest an obvious next step: combining GDN with content-aware token retention to improve multi-item retrieval under a fixed cache budget.

\begin{table}[h]
\centering
\small
\caption{S-NIAH accuracy at 4K context length at different cache token budgets.}\label{tab:ruler_niah}
\begin{tabular}{llcccc}
\toprule
\multirow{2}{*}{Model} & \multirow{2}{*}{Task (4K)} &
\multicolumn{4}{c}{Cache Tokens} \\
\cmidrule(lr){3-6}
& & 128 & 256 & 512 & 1024  \\
\midrule
\multirow{3}{*}{Sliding-Window Attention}
& S-NIAH-1 & 0   & 0   & 5.8 & --  \\
& S-NIAH-2 & 0   & 0   & 0   & --  \\
& S-NIAH-3 & 0   & 0   & 0   & --  \\
\midrule

\multirow{3}{*}{LoLCATs}
& S-NIAH-1 & 0   & 2.2 & 8.8 & --   \\
& S-NIAH-2 & 4.2 & 8.2 & 16.6& --   \\
& S-NIAH-3 & 1.6 & 2.2 & 5.8 & --  \\
\midrule

\multirow{3}{*}{Liger-GLA}
& S-NIAH-1 & 0.2 & 2.8 & 0   & --   \\
& S-NIAH-2 & 1.0 & 0.8 & 0   & --   \\
& S-NIAH-3 & 0.8 & 0.6 & 0   & --   \\
\midrule

\multirow{3}{*}{STILL}
& S-NIAH-1 & 1.4 & 10.6& 86.2& 89.0 \\
& S-NIAH-2 & 2.8 & 9.4 & 37.4& 69.8 \\
& S-NIAH-3 & 3.4 & 3.4 & 12.2& 19.4 \\
\midrule

\multirow{3}{*}{GDN (Ours)}
& S-NIAH-1 & 2.6  & 6.0 & 13.20  & 26.2 \\
& S-NIAH-2 & 2.0  & 5.0  & 16.0 &  34.2\\
& S-NIAH-3 & 1.8 & 4.2 &  13.8 & 28.4\\

\bottomrule
\end{tabular}%
\end{table}

\begin{table*}[h]
\centering
\caption{Extended RULER benchmark at 4K context length.}
\vskip -5pt
\resizebox{\linewidth}{!}{
\begin{tabular}{lclllllllll}
\toprule
\textbf{Model} &
\textbf{Cache Tokens} &
\textbf{MK} &
\textbf{MQ} &
\textbf{MV} &
\textbf{CWE} &
\textbf{FWE} &
\textbf{HQA} &
\textbf{SQA} &
\textbf{VT} &
\textbf{Avg} \\
\midrule

Mamba-2 & 0 &
19.9 & 49.1 & 35.0 & 32.9 & 76.6 & 31.8 & 35.5 & 76.5 & 44.7 \\
LoLCATs & 896 &
4.9 & 3.3 & 3.6 & 3.0 & 13.6 & 14.2 & 14.0 & 0.7 & 7.2 \\
Liger-GLA & 896 & 0 & 0 & 0 & 0.2 & 0.7 & 6.4 & 9.1 & 0 & 2.1 \\
LoLA & 896 &
19.5 & \textbf{67.6} & \textbf{65.0} & 45.9 & 51.3 & 24.2 & 53.9 & \textbf{85.2} & 45.2 \\

STILL & 896 &
21.5 & 52.2 & 51.3 & 36.7 & 76.3 & 32.2 & 49.6 & 63.7 & \textbf{47.9} \\

\midrule
Ours (GDN) & 128 & 1.33 \stds{0.00} & 2.40 \stds{0.03} & 2.38 \stds{0.04} & 34.44 \stds{3.39} & 60.04 \stds{2.37} & 21.27 \stds{0.18} & 23.49 \stds{0.36} & 2.65 \stds{0.04} & 18.50 \stds{0.55} 
\\
Ours (GDN) & 896 & \textbf{25.36} \stds{0.06} & 22.62 \stds{0.09} & 25.38 \stds{0.02} & \textbf{80.49} \stds{3.86} & \textbf{81.67} \stds{1.64} & \textbf{38.60} \stds{0.31} & \textbf{68.37} \stds{0.54} & 37.68 \stds{0.56} & \textbf{47.52} \stds{0.66} 
 \\
\bottomrule
\end{tabular}}
\label{tab:ruler_extended_4k}
\end{table*}

\newpage
\subsection{Results for all Qwen models}\label{app:qwen-scaling-detailed}

\begin{table}[h]
\centering
\caption{Scaling result for all Qwen3 models.}
\label{tab:qwen_scaling}
\resizebox{\linewidth}{!}{
\begin{tabular}{lllllllll|l}
\toprule
Model & Attention & PiQA & ARC-e & ARC-c & HellaSwag & WinoGrande & MMLU & Lambada & Avg \\
\midrule
Qwen3-0.6b & Full & 70.24 & 65.74 & 38.31 & 53.69 & 58.56 & 52.48 & 53.68 & 56.10 \\
 & GDN & 69.89 \stds{0.02} & 65.64 \stds{0.01} & 38.25 \stds{0.06} & 52.53 \stds{0.04} & 57.93 \stds{0.00} & 47.22 \stds{0.11} & 47.77 \stds{0.04} & 54.18 \stds{0.02} \\ \midrule
Qwen3-1.7b & Full & 75.90 & 73.36 & 45.05 & 66.40 & 64.40 & 62.61 &  63.23 & 64.42  \\
 & GDN & 75.55 \stds{0.02} & 73.57 \stds{0.00} & 44.94 \stds{0.03} & 65.34 \stds{0.04} & 64.64 \stds{0.00} & 57.78 \stds{0.11} & 56.50 \stds{0.11} & 62.62 \stds{0.01} \\
 \midrule
Qwen3-4b & Full & 78.02 & 79.21 & 51.62 & 73.67 & 70.72 & 72.98 & 69.09 & 70.76  \\
 & GDN & 78.27 \stds{0.02} & 78.98 \stds{0.01} & 51.96 \stds{0.05} & 72.81 \stds{0.01} & 70.48 \stds{0.00} & 67.52 \stds{0.13} & 62.68 \stds{0.08} & 68.96 \stds{0.03}  \\
\midrule
Qwen3-8b & Full & 79.16 & 81.73 & 56.14 & 78.69 & 72.93 & 76.94 &   70.95 & 73.79 \\
 & GDN & 79.29 \stds{0.04} & 81.86 \stds{0.02} & 56.48 \stds{0.00} & 77.84 \stds{0.03} & 71.98 \stds{0.00} & 71.10 \stds{0.13} & 64.64 \stds{0.12} & 71.88 \stds{0.02}\\
\midrule
Qwen3-14b & Full & 80.47 & 83.33 & 58.96 & 81.27 & 74.19 & 80.46 & 74.33 & 76.14 \\
 & GDN & 80.36 \stds{0.00} & 83.36 \stds{0.04} & 59.36 \stds{0.03} & 80.73 \stds{0.07} & 74.35 \stds{0.00} & 74.63 \stds{0.03} & 67.93 \stds{0.12} & 74.39 \stds{0.02} \\
\midrule
Qwen3-32b & Full & 81.01 & 84.39 & 60.84 & 82.62 & 72.93 & 81.73 & 66.97 & 75.78 \\
 & GDN & 81.08 \stds{0.05} & 84.34 \stds{0.02} & 60.78 \stds{0.03} & 81.74 \stds{0.03} & 73.09 \stds{0.00} & 76.02 \stds{0.17} & 63.04 \stds{0.17} & 74.30 \stds{0.01}\\
\bottomrule
\end{tabular}
}
\end{table}

\newpage
\section{Hyperparameters and Implementation Details}\label{app:hyperparams}
We fixed the experimental hyperparameter for all the experiments. We did not run hyperparameter sweep and fix the same setting for all the base model. Our main motivation is that good approximation should not require extensive training and thus should not depend much on training hyperparameter. 
The only difference is that LoRa experiments use a smaller learning rate than other runs. We list these training details below in Table~\ref{tab:hyperparameters}. We run all the experiments with 3 random seed, fixing the seed to \texttt{1, 2, 3}.

To make the implementation details clear, we also provide the shape of all the newly introduces tensors and mention any initialization that we use. It absent, default pytorch initialization is used. We use the following notations:
\begin{itemize}
    \item Hidden state: $h$
    \item Hidden state feature dimension: $\text{dim}_h$
    \item Head dimension: $\text{dim}_{{head}}$
    \item Number of query heads: $\text{head}_q$
    \item Number of key and value heads: $\text{head}_{kv}$
    \item Output dimension: $\text{dim}_{out} = \text{dim}_{{head}} \cdot \text{head}_q$
\end{itemize}

\texttt{FLA} kernels assumes gates to be in Log scale for stability. Thus on practice we use \texttt{LogSigmoid} activation for the $\log\alpha$. We use \texttt{HedgehogFeatureMap} from the \texttt{FLA} library to compute $\phi_q$ and $\phi_k$ projections. We use \texttt{ShortConvolution}  from the \texttt{FLA} library to compute short convolutions. We apply them after RoPE and before $L2$ normalization. For LoRA we use \texttt{PEFT} library.

\begin{table}[h]
    \caption{All the hyperparameters used for our experimnets.}
    \label{tab:hyperparameters}
\begin{center}
\begin{adjustbox}{max width=\textwidth}
        \begin{tabular}{@{}l|c@{}}
            \toprule
            \multirow{1}{*}{\textbf{Hyperparameter}}  & \\
            \midrule
\multicolumn{2}{c}{\textbf{Optimization}}  \\
learning rate & 2e-3 \\
learning rate (LoRA training) & 5e-4 \\
lr scheduler & Cosine\\
lr warmup & linear \\
lr warmup steps & 100 \\
optimizer & AdamW \\
weight decay & 0 \\
gradient clipping & 1\\
training steps & 2500 \\
\midrule
\multicolumn{2}{c}{\textbf{Training Data}} \\
batch size & 1\\
sequence len & 4096 \\
Dataset & DCLM-Edu \\
\midrule
    \multicolumn{2}{c}{\textbf{Architecture details: Pure Linear (Tables~\ref{tab:downstream_pure},\ref{tab:downstream_pure_swa})}} \\
Gate projection & $y = h W_{\alpha}^{r_1}W_{\alpha}^{r_2} + b_{\alpha}$ \\
& $W_{\alpha}^{r_1} \in \mathbb{R}^{\text{dim}_h\times 16}$,  $W_{\alpha}^{r_2} \in \mathbb{R}^{16\times \text{dim}_{out}}$\\
Gate activation & \texttt{Sigmoid} \\
Gate Initialization & $b_{\alpha} = 1$\\
Beta Projection & $y = h W_{\beta} + b_{\beta}$ \\
& $W_{\beta} \in \mathbb{R}^{\text{dim}_h \times \text{head}_q}$\\
Beta Activation & \texttt{Sigmoid}\\
Beta Initialization & $b_{\beta} = -1$\\
Hedgehog features size & $ \mathbb{R}^{\text{dim}_{head}\times \text{dim}_{head}}$ \\
\midrule
\multicolumn{2}{c}{\textbf{Architecture details: Short Conv (Table~\ref{tab:downstream_sc_lora_swa})}} \\
kernel size & 8\\
bias & True \\
activation & None \\
\midrule
\multicolumn{2}{c}{\textbf{Architecture details: LoRA (Table~\ref{tab:downstream_sc_lora_swa})}} \\
Rank & 8\\
$\alpha$ & 16 \\
Dropout & 0 \\
Target Modules & k\_proj, v\_proj, q\_proj, o\_proj \\
\hline
\bottomrule
\end{tabular}
\end{adjustbox}
\end{center}
\end{table}

\newpage
\section{Broader Impact}\label{App:broader_impact}
Efficient long‑context inference can reduce serving cost and memory footprint, enabling broader access to long‑context capabilities. At the same time, cheaper inference can increase deployment scale and associated energy use. Our work does not introduce new training data or new model capabilities beyond the underlying backbones; it provides an efficient conversion recipe intended for research and deployment settings.

\end{document}